\documentclass{article}
\usepackage{etoolbox}
\usepackage[bottom]{footmisc}  %

\newcommand{\ourtitle}{%
Feature Responsiveness Scores:\\Model-Agnostic Explanations for Recourse 
}

\newcommand{\repourl}{%
\href{https://github.com/ustunb/reachml}{GitHub}}

\providetoggle{blind}
\providetoggle{splitpages}
\providetoggle{workingversion}
\settoggle{blind}{false} %
\settoggle{workingversion}{false} %
\settoggle{splitpages}{false} %

\iftoggle{splitpages}{%
\let\oldsection\section
\renewcommand\section{\clearpage\oldsection}
}{}

\iftoggle{workingversion}{
\usepackage[textwidth=1in,textsize=scriptsize,color=green!30]{todonotes}
}{
\usepackage[disable]{todonotes}
}

\usepackage{ext/iclr2025_conference,times}
\usepackage{algorithm,algpseudocode}
\iclrfinalcopy
\title{\ourtitle{}
}
\author{%
Harry Cheon\\UC San Diego
\And 
Anneke Wernerfelt\\Haverford College
\And
Sorelle A. Friedler\\Haverford College
\And 
Berk Ustun\\UC San Diego
}
\usepackage[toc,page,header]{appendix}
\usepackage{minitoc}

\usepackage[utf8]{inputenc} 
\usepackage[T1]{fontenc}
\usepackage{environ,comment}
\usepackage{amsfonts,bm,bbm,courier,nicefrac,microtype} 
\usepackage{amsmath,amsthm,amssymb}
\usepackage{array,booktabs,tabularx,multirow,colortbl,hhline,tabularray}
\UseTblrLibrary{booktabs}
\usepackage{wrapfig}
\usepackage{soul}
\usepackage{eqnarray}
\usepackage[inline]{enumitem}
\usepackage{xifthen}%
\usepackage{ifthen}
\usepackage{thmtools}
\usepackage{xspace}
\usepackage{svg}

\usepackage{titletoc}

\usepackage{graphicx,xcolor,placeins}
\usepackage[font={footnotesize},labelfont={bf}]{caption}
\usepackage{subcaption}
\usepackage[most]{tcolorbox}

\definecolor{good}{HTML}{BAFFCD}
\definecolor{bad}{HTML}{FFC8BA}
\definecolor{neutral}{HTML}{A2C7DB}

\usepackage{hyperref}
\hypersetup{colorlinks=true, urlcolor=blue, linkcolor=blue, citecolor=blue, linkbordercolor=blue, pdfborderstyle={/S/U/W 1}}

\renewcommand{\arraystretch}{1.2}
\newcolumntype{H}{>{\setbox0=\hbox\bgroup}c<{\egroup}@{}}
\newcolumntype{R}[1]{>{\raggedright\arraybackslash}p{#1}}
\newcommand{\textheader}[1]{{\bfseries{#1}}}
\newcommand{\cell}[2]{\setlength{\tabcolsep}{0pt}\begin{tabular}{#1}#2 \end{tabular}}

\usepackage{pifont}

\usepackage{natbib}
\setcitestyle{numbers,comma,square,sort}

\setitemize{leftmargin=1em,topsep=0.1em,parsep=0.5pt,}
\setlist[enumerate]{leftmargin=*, label= {\arabic*.}, itemsep=0.5em}

\usepackage[capitalise]{cleveref}

\newlist{constraints}{enumerate}{3}
\setlist[constraints]{label={\arabic*.},leftmargin=*}

\newtheorem{theorem}{Theorem}

\newtheorem{remark}[theorem]{Remark}

\newtheorem{fact}[theorem]{Remark}

\theoremstyle{definition}
\newtheorem{definition}{Definition}

\newcommand{\yes}{\ding{51}}%
\newcommand{\no}{\ding{55}}%

\usepackage{algorithm}
\usepackage{algpseudocode}

\algnewcommand{\alginput}[2]{\Statex{Input:~#1}\Comment{#2}}
\algnewcommand\algorithmicinput{\textbf{Input}}
\renewcommand\algorithmicensure{\textbf{Output}}
\algnewcommand\algorithmicinitialize{\textbf{Initialize}}
\algnewcommand\algorithmicbigstep{\textbf{Step}}
\algnewcommand\INPUT{\item[\algorithmicinput]}
\algnewcommand\INITIALIZE{\item[\algorithmicinitialize]}
\algrenewcommand\algorithmiccomment[2][]{#1\hfill{\color{gray}\textit{\scriptsize{#2}}}}

\usepackage{todonotes}

\DeclareMathOperator*{\argmin}{argmin}

\newcommand{\indic}[1]{\mathbbm{1}[#1]}

\renewcommand{\Pr}[1]{\textnormal{Pr}(#1)}

\newcommand{\B}{\{0,1\}}
\newcommand{\R}{\mathbb{R}}
\newcommand{\Z}{\mathbb{Z}}
\newcommand{\intrange}[1]{[#1]}
\newcommand{\defeq}[0]{:=}

\newcommand{\cp}[1]{{\scriptsize{#1}}}
\newcommand{\textds}[1]{{\footnotesize\texttt{{#1}}}} %
\newcommand{\textfn}[1]{{\footnotesize\texttt{{#1}}}} %
\newcommand{\textmn}[1]{{\footnotesize\textsf{{#1}}}} %

\newcommand{\optpar}[2]{\ifthenelse{\isempty{#2}}{#1}{#1({#2})}}

\newcommand{\vecvar}[1]{\bm{#1}}

\renewcommand{\st}{\textnormal{s.t.}}

\newcommand{\xb}{\bm{x}}
\newcommand{\xp}[0]{\xb}
\newcommand{\xq}[0]{\xb'}

\newcommand{\X}{\mathcal{X}}
\newcommand{\Y}{\mathcal{Y}}

\newcommand{\ytarget}[0]{\smash{\hat{y}^\textrm{target}}}
\newcommand{\clf}[1]{\optpar{f}{#1}}

\newcommand{\ab}[0]{\vecvar{a}}

\newcommand{\Ajset}[2]{\smash{A_{#2}({#1})}}

\newcommand{\Rset}[1]{\optpar{R}{#1}}

\newcommand{\Rjset}[2]{\smash{\optpar{R_{#2}}{#1}}}
\newcommand{\Rjhat}[2]{\optpar{\hat{R}_{#2}}{#1}}

\newcommand{\rscore}[2]{\smash{\mu_{#2}({#1})}}
\newcommand{\impscore}[3]{%
\ifthenelse{\isempty{#3}}{%
\phi_{{#2}}({#1})%
}{%
\phi_{{#2}}({#1};{#3})%
}%
}

\newcommand{\impscorevec}[2]{%
\ifthenelse{\isempty{#2}}{%
\bm{\phi}({#1})%
}{%
\bm{\phi}({#1};{#2})%
}%
}

\newcommand{\impscorevecfun}[1]{\bm{\phi}_{{#1}}}

\newcommand{\runningaset}[0]{A_j}

\newcommand{\find}[1]{\mathsf{Find1DAction}{(#1)}}
\newcommand{\runningrset}[0]{\Rjhat{}{j}}
\newcommand{\feasible}[1]{\mathsf{CheckFeasibility}{(#1)}}
\newcommand{\sampleaction}[1]{\mathsf{Sample1DAction}{(#1)}}
\newcommand{\sampledownstream}[1]{\mathsf{SampleDownstream}{(#1)}}
\newcommand{\rscorehat}[2]{\hat{\mu}_{#2}({#1})}
\newcommand{\rscoretilde}[2]{\tilde{\mu}_{#2}({#1})}
\newcommand{\success}[0]{S}
\newcommand{\sampsize}[0]{N}

\newcommand{\textroutine}[1]{{\small\textsf{#1}}}

\newcommand{\epsmin}{\varepsilon_\textrm{min}}
\newcommand{\Knogood}{\mathcal{A}_{j}^\textrm{opt}}
\newcommand{\dpos}[2]{\delta_{#1,#2}^{+}}
\newcommand{\dneg}[2]{\delta_{#1,#2}^{-}}
\newcommand{\chgind}[1]{u_{#1}}

\newcommand{\aopt}[1]{\ab_{#1}}

\newcommand{\mipwhat}[1]{\emph{\scriptsize #1}}
\newcommand{\akpos}{a_{k}^+}
\newcommand{\akneg}{a_{k}^-}

\newcommand{\sublevel}{\reflectbox{\rotatebox[origin=c]{180}{$\Rsh$}}}

\newcommand{\dmMetrics}[0]{%
\cell{l}{%
\% Denied\\ %
\sublevel~\% No Recourse\\ %
\sublevel~\% 1-D Recourse\\ %
\sublevel~\% $n$-D Recourse%
}%
}%

\newcommand{\dmeMetrics}[0]{
    \cell{l}{
        \% Presented with Explanations\\%
        \sublevel~\% All Features Unresponsive\\%
        \sublevel~\% At Least 1 Feature Responsive\\%
        \sublevel~\% All Features Responsive\\%
        \sublevel~\# Features Highlighted
    }
}

\newcommand{\adddatainfo}[5]{\cell{l}{\textds{#1} \\ $n={#2}$ \\ $d={#3}$ \\ {#5}}}
\newcommand{\ficoinfo}[0]{\adddatainfo{heloc}{5,842}{43}{31}{\citet{fico}}}
\newcommand{\germaninfo}[0]{\adddatainfo{german}{1,000}{36}{9}{\citet{dua2019uci}}}
\newcommand{\givemecreditinfo}[0]{\adddatainfo{givemecredit}{120,268}{23}{13}{\citet{data2018givemecredit}}}

\newcommand{\adddatainfolong}[5]{%
\cell{l}{%
\textds{#1}\\
$n={#2}$\\
$d={#3}$ features\\
$d_A={#4}$ mutable\\
{#5}%
}%
}

\newcommand{\ficoinfolong}[0]{\adddatainfolong{heloc}{5,842}{43}{31}{\citet{fico}}}
\newcommand{\germaninfolong}[0]{\adddatainfolong{german}{1,000}{36}{9}{\citet{dua2019uci}}}
\newcommand{\givemecreditinfolong}[0]{\adddatainfolong{givemecredit}{120,268}{23}{13}{\citet{data2018givemecredit}}}

\newcommand{\lime}{{\textmn{LIME}}}
\newcommand{\shap}{{\textmn{SHAP}}}
\newcommand{\shapaa}[0]{\textmn{SHAP-AW}}
\newcommand{\limeaa}[0]{\textmn{LIME-AW}}
\newcommand{\resp}{{\textmn{RESP}}}

\newcommand{\LR}{{$\mathsf{LR}$}}
\newcommand{\RF}{{$\mathsf{RF}$}}
\newcommand{\XGB}{{$\mathsf{XGB}$}}

\newcommand{\regular}{All Features}
\newcommand{\filtered}{Actionable Features}

\definecolor{bad}{HTML}{ff8080}
\definecolor{good}{HTML}{BAFFCD}
\definecolor{recourse}{HTML}{8FAADC}

\algnewcommand\algorithmicforeach{\textbf{for each}}
\algdef{S}[FOR]{ForEach}[1]{\algorithmicforeach\ #1\ \algorithmicdo}

\newcommand\blfootnote[1]{%
  \begingroup
  \renewcommand\thefootnote{}\footnote{#1}%
  \addtocounter{footnote}{-1}%
  \endgroup
}

\begin{document}
\doparttoc
\faketableofcontents
\maketitle
\blfootnote{\vspace{-2em}This is an extended version of the paper accepted at ICLR 2025.}
\begin{abstract}
Consumer protection rules require companies that deploy models to automate decisions in high-stakes settings to explain predictions to decision subjects.
These rules are motivated, in part, by the belief that explanations can promote \emph{recourse} by revealing information that decision subjects can use to contest or overturn their predictions. 
In practice, companies provide individuals with a list of principal reasons based on feature importance derived from methods like SHAP and LIME.
In this work, we show how common practices can fail to provide recourse and
propose to highlight features based on their \emph{responsiveness}---the probability that a decision subject can attain a target prediction through an arbitrary intervention on the feature.
We develop efficient methods to compute responsiveness scores for any model and actionability constraints.
We show that standard practices in lending can undermine decision subjects by highlighting unresponsive features and explaining predictions that are fixed.
\end{abstract}

\section{Introduction}
\label{Sec::Introduction}

Machine learning models routinely automate and support decisions in consumer finance~\citep[][]{hurley2016credit}, employment~\citep[][]{bogen2018help,raghavan2020mitigating}, and public services~\citep{wykstra2020government,eubanks2018automating,gilman2020poverty}. In these domains, companies are increasingly required to provide explanations to decision subjects who receive adverse outcomes (e.g., denied a loan)~\citep{cfpb2024regB, AIBOR, OnlineCivilRightsAct, EUAIAct}. 
In the European Union, for example, Article 86 of the AI Act~\citep{EUAIAct} grants individuals a \emph{right to explanation} in ``high risk'' domains ~\citep[see Annex III of][]{EUAIAct}.
In the United States, the \emph{adverse action} provision in the Equal Credit Opportunity Act mandates that lenders provide a list of ``principal reasons'' to consumers who are denied credit~\cite{cfpb2024regB}.  

Explanations are a cornerstone of consumer protection in such domains because they may reveal information that consumers could use to exercise their broader rights~\citep[][]{edwards2017slave}. 
In the European Union, for instance, the right to an explanation in the GDPR is meant to reveal information that consumers could use to contest their decisions or request human review~\citep[][]{kaminski2021right}.
Likewise, in the United States, adverse action notices are meant to support: \emph{anti-discrimination}, by revealing that a prediction was based on protected characteristics; \emph{rectification}, by revealing that a prediction was based on erroneous information; and \emph{recourse}, by revealing how to attain a desired prediction in the future~\citep{taylor1980meeting,selbst2018intuitive}.

Explainability mandates provide companies with substantial leeway on how they build explanations. In practice, companies resort to the path of least resistance, using popular feature attribution methods like SHAP and LIME to report features in \emph{feature-highlighting explanations} for decision subjects. However, we do not yet know to what extent standard practices achieve the goals of explainability mandates. 
This information is necessary to guide efforts in enacting and enforcing explainability mandates, especially considering many are in early stages of development.

\begin{figure}[t!]
    \centering
    \includegraphics[page=1,trim=0.0in 4.8in 0.9in 0.0in,%
    clip,width=\textwidth]{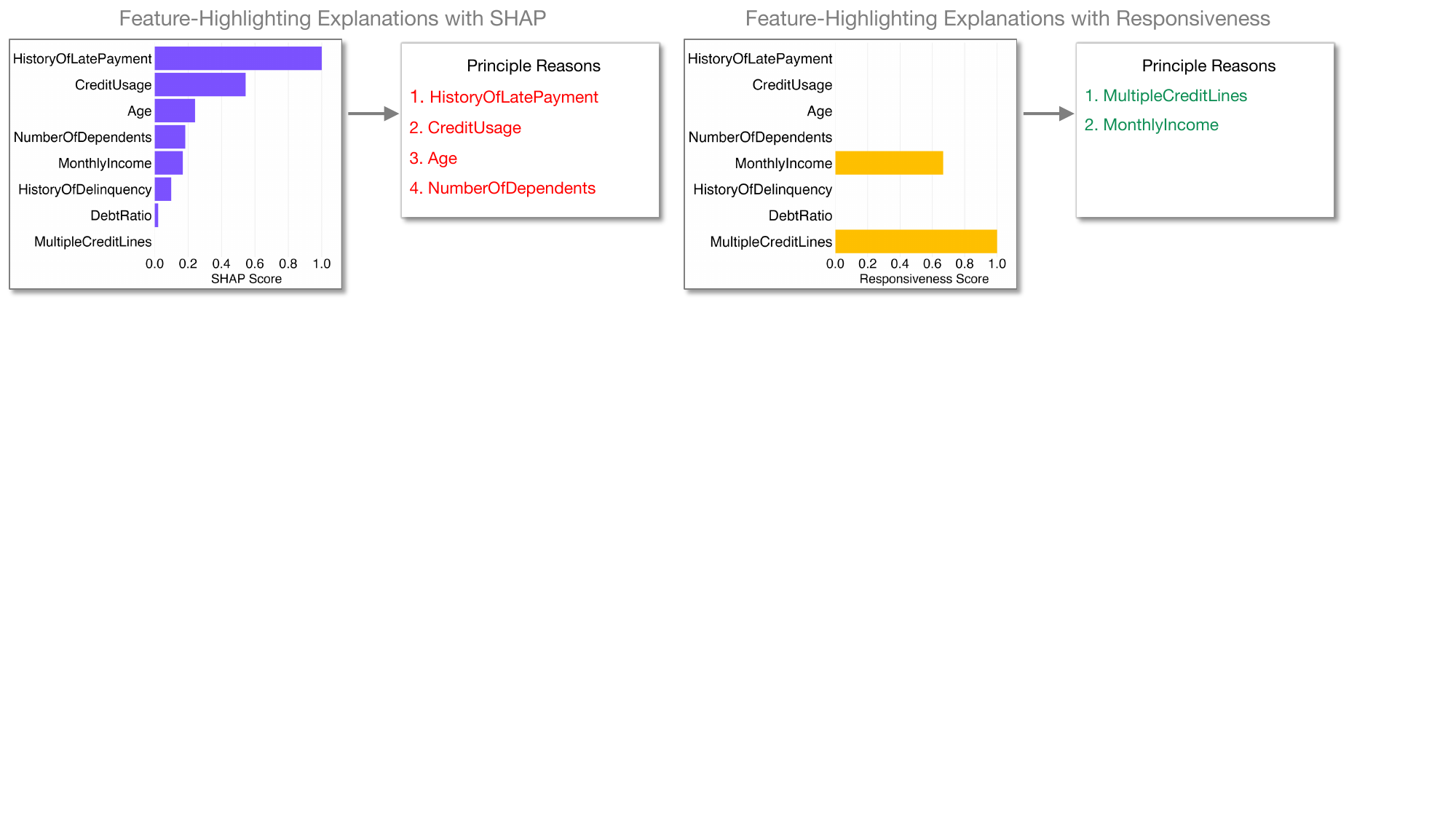}
    \caption{Feature-highlighting explanations for a person denied credit by an XGBoost model on the \textds{\cp{givemecredit}} dataset in~\cref{Sec::Experiments}. We show explanations that highlight up to 4 features with the largest \smash{\shap{}} scores (left) and responsiveness scores (right). As shown, an explanation built with \smash{\shap{}} highlights features that the person cannot change~(e.g., \textfn{\cp{Age}}, \textfn{\cp{HistoryOfLatePayment}}, \textfn{\cp{NumberOfDependents}}) or \emph{unresponsive} (\textfn{\cp{CreditUsage}}, which can be changed but would not lead to a target prediction). In contrast, an explanation built with responsiveness scores highlights \emph{the only} 2 features lead to a desired prediction: \textfn{\cp{MonthlyIncome}} and \textfn{\cp{MultipleCreditLines}}.
    }
    \label{Fig::ResponsivenessOverview}
\end{figure}

We study how explanations can effectively achieve one of their goals---helping consumers attain \emph{recourse}. 
Our main contributions include:
\begin{enumerate}[itemsep=0.1em,topsep=0pt]
    
    \item We identify feature attribution methods can provide \emph{reasons without recourse}---reporting ``important'' features that do not facilitate recourse.

    \item We introduce an approach to highlight features that lead to recourse by measuring \emph{responsiveness}---the probability that an individual can attain a target outcome by intervening on a specific feature.
    
    \item We develop methods to compute feature responsiveness scores for any classification model. Our methods can enforce complex actionability constraints that allow practitioners to control the set of interventions and their downstream effects.

    \item We conduct an extensive empirical study on feature-highlighting explanations in lending. Our results show that standard methods can harm consumers by highlighting immutable and unresponsive features, and that our approach promotes recourse and transparency by highlighting responsive features and flagging predictions that are difficult or impossible to change. %
    
    \item We include a Python library to compute feature responsiveness scores, available on \repourl{} and installable by {\small\texttt{pip install reachml}}.
\end{enumerate}

\paragraph{Related Work}
\begin{wrapfigure}[19]{R}{0.35\linewidth}
\centering
\vspace{-1em}
\resizebox{\linewidth}{!}{
\includegraphics[%
page=2,%
trim=0.0in 5.15in 10.9in 0.0in,%
clip,%
width=\linewidth%
]{figures/figures.pdf}
}
\caption{%
Standard methods for recourse provision return the closest action that leads to a target prediction $\smash{\bm{a}^\textrm{opt}}$.
Our method estimates the proportion of actions on each feature that lead to a target prediction.
Here, $\mu_1 = \tfrac{3}{4}$ and $\mu_2=\tfrac{1}{4}$ because $\xp$ can attain a target prediction through 3/4 actions on $x_1$, or 1/4 actions on $x_2$.}
\label{Fig::ResponsivenessVsRecourse}
\end{wrapfigure}

Our work is related to a stream of methods to explain individual predictions~\citep[][]{ribeiro2016should,lundberg2017unified,marx2019disentangling,kaur2020interpreting}. We identify these methods can inflict harm by providing individuals with \emph{reasons without recourse}. 
We view reasons without recourse as a structural limitation that affects how we operationalize explainability mandates, akin to limitations of explainability that arise due to the multiplicity of predictions~\citep{marx2020multiplicity,watsondaniels2022predictive,black2022model}, the indeterminacy of explanations~\citep[]{brunet2022implications,krishna2022disagreement}, and the potential for fairwashing~\citep{aivodji2019fairwashing,slack2020fooling,goethals2023manipulation}. 

Our goals are aligned with works in algorithmic recourse, in that we seek to provide individuals with information to overturn adverse outcomes~\citep[][]{ustun2019actionable, karimi2020algorithmic,venkatasubramanian2020philosophical}. Many recourse methods are designed to return an \emph{action} that an individual could perform to attain a target prediction. In contrast, our method is designed to estimate the prevalence of actions that lead to a target prediction (see \cref{Fig::ResponsivenessVsRecourse}). 
We construct these estimates through algorithms that sample or enumerate a set of reachable points~\citep{kothari2024prediction}. The resulting approach is model agnostic and can be adapted them to address practical challenges related to causality~\citep[][]{karimi2021interventions,dominguez2022adversarial,galhotra2021explaining} and distributional robustness~\citep{nguyen2023distributionally,pawelczyk2023probabilistically,upadhyay2021towards}.

\section{Problem Statement}
\label{Sec::ProblemStatement}

We consider a classification task where a company uses a model $\clf{}: \X\to \Y$ to predict a label $y \in \Y$ from a set of \emph{features} $\xb = [x_1, x_2, \ldots, x_d] \in \X \subseteq \R^d$.  We focus on tasks where each instance represents a person, and their features encode semantically meaningful characteristics. In such settings, we can assume that features are bounded. In practice, many features are bounded by definition---e.g., a binary feature such as $\textfn{recent\_payment} \in \{0,1\}$. In other cases, we can set loose bounds that apply to all decision subjects---e.g., $\textfn{age} \in [0, 120]$.

We specify the subset of individuals who are entitled to explanations in terms of a \emph{target prediction} $\ytarget \in \Y$. We assume that $\ytarget$ represents a desirable outcome, e.g., $\clf{\xp} = \ytarget{} = 1$ if a person with features $\xp$ will repay their loan. Under these conventions, companies provide an explanation to any person with features $\xp$ such that $\clf{\xp} \neq \ytarget$. Informally, such an explanation would lead to \emph{recourse}~\citep{ustun2019actionable} when it contains specific information for this person to overturn an adverse outcome---e.g., by describing how to change their features to attain a point $\xq$ such that $\clf{\xq}=\ytarget$.

\paragraph{Feature-Highlighting Explanations}

Companies comply with explainability mandates by building a \emph{feature-highlighting explanation}---i.e., a list that contains the most important features for a specific prediction~\citep[][]{barocas2020hidden}.
In practice, they derive the importance of each feature from post-hoc explainability methods such as LIME or SHAP. In what follows, we refer to this procedure as a \emph{feature attribution method} and define it below.
\begin{definition}
\label{Def::FeatureExplainer}
Given a model $\clf{}: \X \to \Y$ and a point $\xp \in \X$, a \emph{feature attribution method} is a function $\impscorevecfun{}: \X \to \R^d$ that returns a vector of \emph{feature importance scores} $\impscorevec{\xp}{\clf{}} := [\impscore{\xp}{1}{\clf{}},\ldots, \impscore{\xp}{d}{\clf{}}]$. The score for each feature $\impscore{\xp}{j}{\clf{}}$ reflects its relative contribution towards the prediction. In what follows, we write $\impscorevec{\xp}{}$ instead of $\impscorevec{\xp}{\clf{}}$ when $\clf{}$ is clear from context.
\end{definition}
We use the function $\impscorevecfun{}: \X \to \R^d$ to represent common approaches to extract feature importance scores from local explanations:
\begin{itemize}[itemsep=0.1em,topsep=0.0em]

    \item \emph{Local Surrogates}~\citep[see e.g.,][]{ribeiro2016should,zhou2021s,zafar2019dlime}, which explain the prediction of a model $\clf{}$ at a point $\xp$ by fitting a surrogate model to approximate the decision boundary of $f$ near $\xp$. Given the surrogate model, we use its parameters to determine the importance scores for each feature: $\impscore{\xp}{j}{}$. %
    
    \item \emph{Shapley Values}~\citep[see e.g.,][]{lundberg2017unified,jethani2021fastshap,fumagalli2024shap}, which cast the features of a model $\clf{}$ as ``players'' in a cooperative game. Each score $\impscore{\xp}{j}{}$ reflects the marginal contribution of feature $j$ towards the prediction $\clf{\xp}$. %
\end{itemize}
Scores from these methods indicate relative importance due to the following properties:
\begin{itemize}[leftmargin=*,itemsep=0.25em,topsep=0.0pt]

\item \emph{Relevance}: Changing a feature $j$ with $\impscore{\xp}{j}{} = 0$ does not affect the model (i.e., the feature can be dropped)~\cite[see e.g., the missingness axiom in][]{lundberg2017unified}.

\item  \emph{Strength}: Given two features $j, k \in \intrange{d}$ such that $|\impscore{\xp}{j}{}| > |\impscore{\xp}{k}{}|$, feature $j$ has a stronger contribution to the prediction than feature $k$~\cite[see e.g.,][]{neuhof2024ranking}.
\end{itemize}
Given a list of top-scoring features, companies convert the list into a natural language explanation~\citep[e.g., a reason code][]{experianReasonCodes,fico2022reasoncode}. In doing so, they can claim that they have met regulatory requirements by providing tailored and accessible explanations to decision subjects.

\paragraph{Reasons without Recourse}
\begin{wraptable}[18]{R}{0.35\textwidth}
\resizebox{\linewidth}{!}{
\begin{tabular}{ccrlc}
\multicolumn{2}{c}{\cell{>{\bfseries}c}{Feature\\Values}} &
\multicolumn{2}{c}{\cell{>{\bfseries}c}{Label \\Counts}} &
\multicolumn{1}{c}{\cell{>{\bfseries}c}{Best\\Predictions}} \\
\cmidrule(lr){1-2}\cmidrule(lr){3-4}\cmidrule(lr){5-5}
\textfn{age\,$\geq$\,60} & \textfn{savings\,$\geq$\,50K} &
 $n_0$ & 
 $n_1$ & 
 $\clf{x_1,x_2}$ \\
\cmidrule(lr){1-2}\cmidrule(lr){3-4}\cmidrule(lr){5-5}
 0 & 0 & 
 40 & 10  &
 0 \\
\cmidrule(lr){1-2}\cmidrule(lr){3-4}\cmidrule(lr){5-5}
 0 & 1 & 
 10 & 30 &
 1 \\
\cmidrule(lr){1-2}\cmidrule(lr){3-4}\cmidrule(lr){5-5}
 \rowcolor{bad} 1 & 0 & 
 20 & 10 &
 $0$ \\
\cmidrule(lr){1-2}\cmidrule(lr){3-4}\cmidrule(lr){5-5}
 \rowcolor{bad} 1 & 1 & 
 30 & 10 &
 $0$ %
\end{tabular}
}
\caption{Stylized classification task where the most accurate model assigns \emph{fixed predictions} due to an immutable feature \textfn{\cp{age\,$\geq$\,60}}. We train a model to predict $y = \textfn{\cp{repayment}}$ $(x_1, x_2)$ = (\textfn{\cp{age\,$\geq$\,60}}, \textfn{\cp{savings\,$\geq$\,50K}}) on a dataset with $n_0$ negative labels and $n_1$ positive labels. Here, individuals with \textfn{\cp{age}}\,$\geq$\,\textfn{\cp{60}} = 1 are assigned a fixed prediction, as $f(x_1, x_2) = 0$ for all reachable points $\{(1,0), (1,1)\}$.}
\label{Tab::NoRecourse}
\end{wraptable}

One of the limitations of feature-highlighting explanations based on importance scores is that they can highlight features that do not provide recourse. We refer to this phenomenon as \emph{reasons without recourse}. Feature attribution methods can provide reasons without recourse due to two key blind spots:
\begin{itemize}
\item \emph{Ignorance of Counterfactual Behavior}: They can assign high scores to features that are not responsive---i.e., changing the feature does not change the prediction (see e.g., \citet{bilodeau2024impossibility}).

\item \emph{Ignorance of Actionability}: They do not account for how individuals can change their features. This can lead them to assign high scores immutable features.
\end{itemize}
In practice, these failure modes can render explainability mandates counterproductive.
Explanations can highlight the ``wrong'' features if there are other features that the decision subject can change to attain recourse.
In the worst case, the explanation can be providing reasons for a fixed prediction---i.e., $\clf{\xp}$ remains the same under all feasible actions (see \cref{Tab::NoRecourse}).

\section{Measuring Feature Responsiveness}
\label{Sec::ResponsivenessScores}

\newcommand{\vb}[0]{\bm{v}}
\newcommand{\dnstream}[2]{\bm{\delta}({#1,#2})}
\newcommand{\dnstreamDet}[2]{\bm{\delta}^\mathrm{det}({#1,#2})}
\newcommand{\dnstreamSto}[2]{\bm{\varepsilon}({#1,#2})}
\newcommand{\Vjset}[2]{V_{#2}({#1})}
\newcommand{\Vset}[1]{V({#1})}
\newcommand{\dnstreamRV}[0]{E}
\newcommand{\dnstreamProb}[0]{P_{\xb,\ab}}
\newcommand{\newXa}[2]{X^\textrm{reach}_{j}({#1,#2})}
\newcommand{\newX}[1]{X^\textrm{reach}_{j}({#1})}

Our goal is to measure the \emph{responsiveness} of features for decision subjects---i.e., how often the prediction of a model changes after they intervene on a given feature.
However, individuals can only intervene on features in certain ways, and these changes may also have downstream effects on other features.
For instance, changing \textfn{married} from $0 \to 1$ will change \textfn{single} from $1 \to 0$.
Or increasing \textfn{years\_account\_history} will also cause a commensurate increase in \textfn{age}.
In light of these challenges, we describe how decision subjects can intervene on individual features through an \emph{intervention model}:
\begin{definition}
\label{Def::ResponseModel}
Given a point $\xp$, we assume that individual who intervene feature $j$ will move to a new point $\xq$ where:
\[
     \xq = \xp + \ab + \dnstreamSto{\xb}{\ab}
\]
Here:
\begin{itemize}
    \item $\ab = [a_1,\ldots, a_d] \in \Ajset{\xp}{j}$ is an \emph{action} that represents the deterministic components of the intervention; this includes the change in feature $j$ and its deterministic downstream effects. We assume that $a_j \neq 0$ and refer to the set of all possible actions as the \emph{action set} $\Ajset{\xp}{j}$.
    
    \item $\dnstreamSto{\xb}{\ab}$ is a sample of the random variable that represents the stochastic component of the intervention. The sample $\dnstreamSto{\xp}{\ab}$ is drawn from the probability distribution $\dnstreamProb$.
\end{itemize}

For a fixed action $\ab$, $\xq$ is a realization of a random variable. In what follows, we denote the modified features associated with a specific action $\ab$ as the random variable $\newXa{\xp}{\ab}$. When the action itself is also random, we denote the resulting random variable as $\newX{\xp}$.
\end{definition}

\paragraph{On Specification}

This model allows practitioners to specify feasible interventions and their deterministic and stochastic effects.
They can define feasible interventions and deterministic downstream effects in the action set $\Ajset{\xp}{j}$ through \emph{actionability constraints}. These include \emph{separable constraints} that only pertain to one feature (e.g., bounds and monotonicity) and \emph{joint constraints} across multiple features (e.g., deterministic downstream effects). As shown in \cref{Table::ActionabilityConstraintsCatalog}, we can elicit these constraints from human experts in natural language and convert them into equations that we can embed into optimization problems to enforce actionability~\citep[e.g., to search for recourse actions][]{ustun2019actionable,kothari2024prediction}.

Practitioners can define stochastic effects through the conditional probability distribution $\dnstreamProb$. 
This distribution can represent probabilistic causal effects of interventions. For example, we can model the impact of employment on health insurance with $\varepsilon_{\textfn{has\_insurance}} \sim a_{\textfn{employed}} \cdot \mathrm{Bernoulli}(\lambda)$, where whether one has health insurance largely follows their employment status.
Similarly, we can model random fluctuations in features that occur between successive predictions. For instance, we can model the fluctuation in the number of bank transactions per month with $\varepsilon_{\textfn{n\_transactions}} \sim \mathrm{Pois}(\lambda)$. In both cases, $\dnstreamProb$ denotes the corresponding probability mass functions.

\begin{table}[b]
   \centering
   \resizebox{0.95\linewidth}{!}{
   \begin{tabular}{lHHHR{0.49\linewidth}lR{0.37\linewidth}}
        \textbf{Requirement} & 
         & 
         & 
        &
        \textbf{Example} &
        \textbf{Features} & 
        \textbf{Actionability Constraint}
        \\
   \midrule

   Immutability & 
   \yes & \yes & \no &
   \cell{l}{$\textfn{age}$ cannot change} &
   $x_j = \textfn{age}$ &
   $v_j = 0$ \\
   \midrule

    Monotonicity & 
   \yes & \yes & \no &
   $\textfn{recent\_payment}$ can only increase & 
   $x_j = \textfn{recent\_payment}$ &
   $v_j \geq 0$ \\
   \midrule
   
    Integrality & 
    \yes & \yes & \yes &
    $\textfn{late\_payments}$ must be positive integer $\leq 12$ &
    $x_j = \textfn{late\_payments}$ &
    $v_j \in \Z^+ \cap [0 - x_j, 12 - x_j]$  \\
    \midrule
    
    \cell{l}{Encoding\\Validity} & 
    \no & \yes & \yes &
    \cell{l}{%
    preserve one-hot encoding of categorical\\ feature 
    \textfn{housing} $\in \{\textfn{own}, \textfn{rent}, \textfn{other}\}$%
    } & 
    \cell{l}{%
    $x_k = \indic{\textfn{housing=own}}$\\ 
    $x_l = \indic{\textfn{housing=rent}}$\\
    $x_m = \indic{\textfn{housing=other}}$
    } &
    \cell{l}{
    $v_j + x_j \in \{0,1\} ~\textrm{for}~ j \in \{k,l,m\}$ \\
    $\sum_{j \in \{k,l,m\}} v_j + x_j = 1$} \\ 
    \midrule

    \cell{l}{Logical\\Implication} & 
    \no & \yes & \yes &
    \cell{l}{%
    if $\textfn{has\_savings\_account}=\textfn{TRUE}$\\
    then $\textfn{savings\_balance} \geq 0 $\\ 
    else $\textfn{savings\_balance} = 0$}  &
    \cell{l}{%
    $x_j = \textfn{has\_savings\_account}$\\ $x_k = \textfn{savings\_balance}$} &
    \cell{l}{%
    $v_j + x_j \in \{0,1\}$\\
    $v_k + x_k \in [0, 10^{12}]$\\
    $v_k + x_k\leq 10^{12}(x_j + v_j)$%
    }  \\ 
    \midrule
    
    \cell{l}{Causal\\Implication} & 
    \no & \no & \yes &
    \cell{l}{if $\textfn{years\_of\_account\_history}$ increases\\then $\textfn{age}$ will increase commensurately} &  
    \cell{l}{$x_j = \textfn{years\_of\_account\_history}$\\ $x_k = \textfn{age}$} &
    \cell{l}{
    $x_j + v_j \leq x_k + \delta_k$ \\
    $\delta_k \in [0, 100]$
    }\\
    \bottomrule
    \end{tabular}
    }
    \caption{Examples of actionability constraints on semantically meaningful features for a lending task. Each constraint can be expressed in natural language and embedded into an optimization problem using standard techniques in mathematical programming~\cite [see, e.g.,][]{wolsey2020integer}. See \cref{Appendix::ExperimentDetails} for more examples.
    } 
    \label{Table::ActionabilityConstraintsCatalog}
\end{table}

\paragraph{Measuring Responsiveness}
Given our intervention model, we wish to score each feature by the probability that an individual attains a target prediction after performing an arbitrary intervention.
\begin{definition}
\label{Def::ResponsiveScore}
Given a model $\clf{}:\X \to \Y$, a point $\xp \in \X$, a feature $j \in \intrange{d}$, its action set $\Ajset{\xp}{j}$ and the downstream distribution $\dnstreamProb$, the \emph{responsiveness score} of feature $j$ measures the probability that an intervention on feature $j$ attains the target prediction:
\begin{align*}
    \rscore{\xp}{j} :=
    \mathrm{Pr}\Bigl(f\bigl(\newXa{\xp}{\ab}\bigr) = \ytarget{} \mid \ab \in \Ajset{\xp}{j}\Bigr)
\end{align*}
\end{definition}
Here, a score of $\rscore{\xp}{j} = 0$ means that changing feature $j$ cannot achieve $\ytarget$, while $\rscore{\xp}{j} = 1$ means any intervention on $j$ will achieve $\ytarget$.
\paragraph{Benefit for Consumer Protection}
When we construct feature-highlighting explanations using the top-$k$ most responsive features, we reveal the $k$ most promising paths to recourse.
By construction, they are features where arbitrary interventions are most likely to lead to a target prediction.
We make no assumptions on how individuals will change their features beyond actionability constraints specified in $\Ajset{\xp}{j}$. This is because modeling how individuals will intervene on features is not feasible; it cannot be verified a priori.

Contrary to existing methods, our approach provides explanations \emph{only} to individuals with recourse---i.e., we would never provide reasons without recourse. In effect, we can detect instances where providing feature-highlighting explanations can be misleading or harmful by checking if the responsiveness score $\rscore{\xp}{j} = 0$ for all features.
\begin{remark}
    \label{Rem::Recourse}
    Given a model $\clf{}: \X \to \Y$, denote its feature responsiveness scores at the point $\xp \in \X$ as $\rscore{\xp}{1} \ldots \rscore{\xp}{d}$.
    If $\rscore{\xp}{j} = 0$ for all $j\in \intrange{d}$, then either:
        \begin{enumerate}[label={(\alph*)},itemsep=0.1em,topsep=0.1em]
        \item $\clf{}$ assigns a fixed prediction to $\xp$, or \label{SG::RedLightFixed}
        \item $\clf{}$ can only provide recourse to $\xp$ through an intervention on two or more features. \label{SG::RedLightJoint}
        \end{enumerate}
        \label{SG::RedLight}
\end{remark}

According to \cref{Rem::Recourse}, there are two scenarios where $\rscore{\xp}{j} = 0$ for all $j \in \intrange{d}$. We can employ different strategies to mitigate harm in each case. In Case \ref{SG::RedLightFixed}, where individuals receive fixed predictions, we can withhold explanations and notify developers or regulators. In Case \ref{SG::RedLightJoint}, where individuals can only overturn their prediction by intervening on multiple features at the same time, we can include a warning against assuming feature independence.

The responsiveness score $\rscore{\xp}{j}$ depends on the actionability constraints that characterize $\Ajset{\xp}{j}$---i.e., responsiveness scores can change under a different set of constraints. Hence, if constraints are misspecified (e.g., ignoring monotonicity constraints), the scores can lead to misleading conclusions. In tasks where downstream effects are deterministic, we can mitigate this effect by encoding \emph{indisputable constraints} based on feature encodings or physical limits. Then, the corresponding $\rscore{\xp}{j}$ represents an upper bound on the true responsiveness of feature $j$---i.e., decision subjects flagged with a fixed prediction will also have a fixed prediction under more stringent constraints.

\section{Computing Scores with Reachable Sets}
\label{Sec::Computation}

We now introduce an approach to compute responsiveness scores for any model. We compute the responsiveness score of feature $j$ using its reachable set $\Rjset{\xp}{j}$---the set of reachable points through interventions on $j$ (see \cref{Fig::ResponsivenessComputation}). We can generate the reachable set $\Rjset{\xp}{j}$ either by enumerating all possible points---when features are discrete and interventions do not have stochastic effects---or sampling. Given $\Rjset{\xp}{j}$, we can compute the responsiveness score as:
\begin{align}
\rscore{\xp}{j} &:= 
\mathbb{E}_{\xq \sim \newX{\xp}} \bigl[ \indic{\clf{\xq} = \ytarget} \bigr]
\label{Eq::ExpectedResponsivenessReachable}
\end{align}

This approach has several benefits. It is model agnostic; given the reachable sets for each feature, computing the reachable set only requires query access to the model. We can also amortize the cost of generating reachable points by generating the reachable sets \emph{once} and (re)using it to compute responsiveness scores for any model (e.g., during model selection).

In practice, enumerating or sampling reachable points can be challenging. We often need to consider points in regions with little structure---points may have both discrete and continuous dimensions and obey non-convex constraints. Furthermore, when we have stochastic downstream effects, we must check that both interventions and their downstream effects are feasible under actionability constraints. 
We overcome these challenges by casting the generation of reachable points as repeated optimization problems. 

\begin{figure}[!t]
    \centering
    \includegraphics[trim=0.0in 4.8in 4.8in 0.0in,clip,page=4,width=0.8\textwidth]{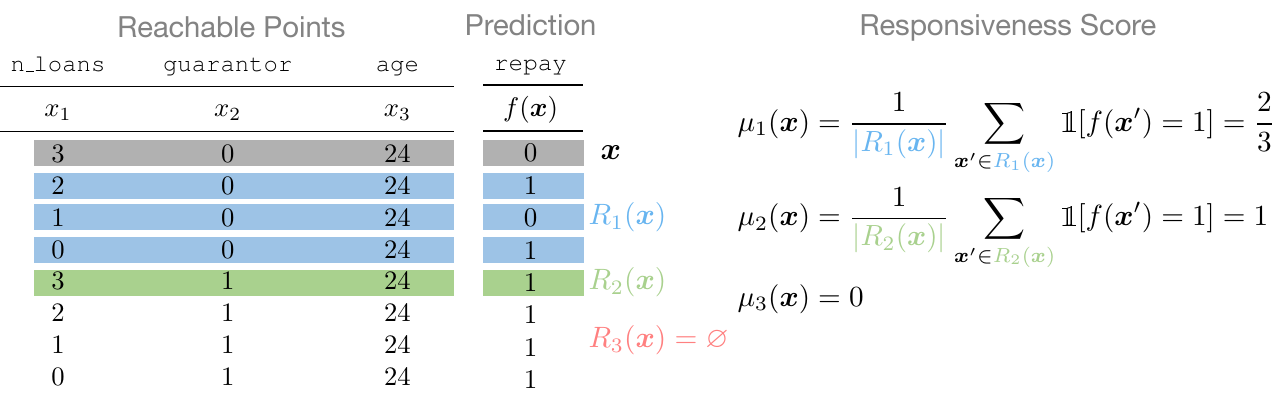}
    \caption{Stylized example showing how to compute responsiveness scores for a classification model with three features \textfn{\cp{n\_loans}}, \textfn{\cp{guarantor}} and \textfn{\cp{age}}. The reachable set $R_j(\xp)$ all points that can be attained from $\xp = (3,0,24)$ by intervening on feature $j$, and $R_3(\xp) = \varnothing$ because \textfn{\cp{age}} is immutable. Given a model $f$, we compute the responsiveness score of each feature by querying its predictions over points in their reachable set $R_j(\xp)$.}
\label{Fig::ResponsivenessComputation}
\end{figure}

\paragraph{Sampling}
\begin{wrapfigure}[15]{R}{0.45\linewidth}
\centering
\vspace{-21pt}
\begin{minipage}{0.45\textwidth}
\begin{algorithm}[H]
\begin{algorithmic}[1]\footnotesize
\Require{$\xp \in \X$}\Comment{point}
\Require{$\Ajset{\xp}{j}$}\Comment{action set for feature $j$}
\Require{$\dnstreamProb$}\Comment{stochastic downstream effect dist. for $j$}
\Require{$\sampsize \in \mathbb{N}$}\Comment{sample size (see \cref{Appendix::SamplingDetails})}
     \Statex $\runningrset \gets \varnothing$
\Repeat
     \State $\ab \gets \sampleaction{\xp, \runningaset}$\label{AlgStep::SampleAction}
     \State $\bm{\varepsilon} \gets \sampledownstream{\xp, \ab, \dnstreamProb}$\label{AlgStep::SampleDownstream}
     \If{$\feasible{\xb, \ab + \bm{\varepsilon}, \runningaset}$}\label{AlgStep::CheckFeasibility}
         \State $\runningrset \gets \runningrset \cup \{ \xp + \ab + \bm{\varepsilon}\}$
     \EndIf
     \Until{$|\runningrset| = \sampsize$}
     \Ensure{$\runningrset$} \Comment{$\sampsize$ reachable points via actions on $j$}
\end{algorithmic}
\caption{{\small Sample Reachable Points}}
\label{Alg::ReachableSetSampling}
\end{algorithm}
\end{minipage}
\end{wrapfigure}
We present a procedure to sample reachable points in \cref{Alg::ReachableSetSampling}. Given a point $\xp$ and an action set $\Ajset{\xp}{j}$, this procedure returns a uniform sample of $\sampsize$ reachable points via rejection sampling.
In \cref{AlgStep::SampleAction}, it calls the $\sampleaction{\xp, \runningaset}$ routine to propose a candidate deterministic change $\ab$ that obeys separable constraints such as bounds and integrality. We then sample its stochastic downstream effect $\bm{\varepsilon}$ in \cref{AlgStep::SampleDownstream}. In \cref{AlgStep::CheckFeasibility}, it then calls the \textsf{CheckFeasibility} routine to check if both the intervention and the downstream effect obey actionability constraints by solving a mixed-integer program. The procedure terminates once it has sampled $\sampsize$ reachable points through interventions on $j$. Given $\Rjset{\xp}{j}$, we can recover an unbiased estimate of the responsiveness score for feature $j$ and a model $\clf{}$ as $\rscorehat{\xb}{j} := \frac{1}{\sampsize{}}\sum_{\xq\in\Rjhat{j}{\xp}} \indic{\clf{\xq} = \ytarget{}}$. We can set the sample size $\sampsize$ to ensure practical guarantees on how reliably we flag fixed predictions (\cref{Rem::Recourse}) as described in \cref{Appendix::SamplingDetails}.
 
\paragraph{Enumeration}
\renewcommand{\runningrset}{R_j}
\begin{wrapfigure}[11]{R}{0.45\linewidth}
\centering
\vspace{-23pt}
\begin{minipage}{0.45\textwidth}
\begin{algorithm}[H]
\begin{algorithmic}[1]\footnotesize
\Require{$\xp \in \X$}\Comment{point}
\Require{$\Ajset{\xp}{j}$}\Comment{action set for discrete feature $j$}
\Statex $\runningrset \gets \varnothing$, $\runningaset \gets \Ajset{\xp}{j}$
\Repeat 
     \State $\ab^{*} \gets \find{\xp, \runningaset}$ %
     \State $\runningrset \gets \runningrset \cup \{ \xp + \ab^{*} \}$ %
     \State $\runningaset \gets \runningaset \setminus \{ \ab^* \}$ \label{AlgStep::RemoveAction}%
\Until {$\find{\xp, \runningaset}$ is infeasible}
\Ensure{$\runningrset$} \Comment{all reachable points via actions in $j$}
\end{algorithmic}
\caption{\small{Enumerate Reachable Points}}
\label{Alg::ReachableSetEnumeration}
\end{algorithm}
\end{minipage}
\end{wrapfigure}
We present a procedure to enumerate a reachable set $\Rjset{\xp}{j}$ in \cref{Alg::ReachableSetEnumeration} for discrete features with deterministic downstream effects. Given the action set, which encodes all actionability constraints (including deterministic downstream effects), the procedure enumerates reachable points for feature $j$ by repeatedly solving the following discrete optimization problem:
\begin{align*}
\label{Eq::Find1DAction}
    \find{\xp, \runningaset} :=  \argmin_{\ab \in \Ajset{\xp}{j}}  \|\ab\|_1
\end{align*}
We formulate $\find{\xp, \runningaset}$ as a mixed-integer program, and update it at each iteration with a ``no good'' constraint to remove previous optima in \cref{AlgStep::RemoveAction} (see \cref{Appendix::FindMIP} for exact formulation). We use each action to add a reachable point to $\Rjset{\xp}{j}$ and use the final set to calculate \emph{exact} responsiveness scores. We adapt a method to enumerate the reachable set for all features from \citet{kothari2024prediction}, but is more tractable as we only enumerate points that can we can attain through interventions on each feature.

\paragraph{Extensions}
\label{Par::Variants}
\newcommand{\rbscore}[2]{\mu^\textrm{robust}_{#2}({#1})}
\newcommand{\wtscore}[2]{\mu^\textrm{cost}_{#2}({#1})}

One of the benefits of reachable sets is that we easily customize scores to meet additional requirements.
One such requirement is \emph{monotonicity}, i.e., if a person is guaranteed a target prediction by increasing (or decreasing) a feature beyond a threshold value. 
In the simplest case, we can account for properties through operations like filtering or weighing (see e.g., \cref{Sec::Demos}). In general, we can construct responsiveness scores that address practical challenges given additional inputs:
\begin{itemize}[itemsep=0.2em]

    \item \emph{Individual Preferences}: Given a cost function that captures the difficulty of actions in each direction, we can highlight features that are easier to change (i.e., least costly $k$ features) using a cost-weighed score: $\wtscore{\xp; \; \mathrm{cost}}{j} = \sum_{\xq \in \Rjset{\xp}{j}} \mathrm{cost}(\xq;\xp) \cdot \indic{\clf{\xq} = \ytarget}$. %
    
    \item \emph{Distributional Robustness}: Given a general reachable set $\Rset{\xp}$ that contains all points that we could reach through interventions on any feature, we highlight features that attain a target prediction regardless of how other features change through the \emph{robust score}:
    $\rbscore{\xp}{j} = \min_{\bm{\delta} \in \Delta_{-j}} \mathbb{E}_{X' \sim \Rjset{\xp}{j}}[\indic{\clf{X'+ \bm{\delta}} = \ytarget{}}],$
    where $\Delta_{-j} := \{\bm{\delta} \in \R^d \mid \delta_j = 0, \|\bm{\delta}\| < \varepsilon\}.$ %

\end{itemize}

\section{Experiments}
\label{Sec::Experiments}

We present an empirical study on the responsiveness of explanations. Our results reveal the limitations of existing feature attribution methods and show how our approach can support recourse and flag fixed predictions. We include details in \cref{Appendix::ExperimentDetails}, and code to reproduce our results on \repourl{}.

\vspace{-0.5em}
\paragraph{Setup}
We work with three publicly available classification datasets from consumer finance. Here, each instance represents a consumer and the label indicates if they will repay a loan. We consider discrete version of each dataset in which we can compute exact responsiveness scores and certify if each person has recourse. Given these datasets, we define \emph{inherent actionability constraints} which reflect indisputable requirements that apply to all individuals (e.g., no changes to immutable attributes, preserve feature encoding, and adhere to deterministic causal effects).

We split each dataset into a training sample (80\%; to train models) and a test sample (20\%; to evaluate out-of-sample performance). We fit models using (1) \emph{logistic regression} (\LR{}), (2) \emph{XGBoost} (\XGB{}), and (3) \emph{random forests} (\RF{}). For each model, we construct a feature-highlighting explanation for each person who is denied credit in the dataset that includes up to \emph{four features}; if all features have a score of 0, we do not present an explanation for that individual. The choice of \emph{up to four} features reflects the recommended number of reasons to show in an adverse action notice by the U.S. Consumer Finance Protection Bureau~\citep[see][]{cfpb2024regBcomment}. We include the top-4 scoring features from the following methods:
\begin{wraptable}[21]{r}{0.435\textwidth}
\vspace{-5pt}
\centering\resizebox{\linewidth}{!}{

\begin{tabular}{lllll}
Dataset & Metrics & \LR{} & \RF{} & \XGB{}\\
\midrule
\ficoinfo{} & \dmMetrics{} & \cell{r}{56.1\%\\\color{bad}{22.2\%}\\41.0\%\\36.8\%} & \cell{r}{58.3\%\\\color{bad}{31.3\%}\\31.7\%\\37.0\%} & \cell{r}{57.0\%\\\color{bad}{53.1\%}\\25.3\%\\21.6\%}\\
\midrule

\germaninfo{} & \dmMetrics{} & \cell{r}{22.9\%\\\color{bad}{7.4\%}\\74.2\%\\18.3\%} & \cell{r}{17.5\%\\\color{bad}{28.6\%}\\48.0\%\\23.4\%} & \cell{r}{22.0\%\\\color{bad}{11.8\%}\\68.2\%\\20.0\%}\\
\midrule

\givemecreditinfo{} & \dmMetrics{} & \cell{r}{24.6\%\\\color{bad}{15.6\%}\\72.4\%\\12.0\%} & \cell{r}{24.7\%\\\color{bad}{0.2\%}\\93.2\%\\6.6\%} & \cell{r}{24.8\%\\\color{bad}{11.5\%}\\76.0\%\\12.5\%}\\
\bottomrule
\end{tabular}
}
\caption{%
  Overview of paths to recourse for individuals who would receive an explanation for each dataset and model.
  We report \emph{\% Denied}, \% of denied individuals; \textcolor{bad}{\emph{\% No Recourse}}, \% of denied with a fixed prediction (i.e., who have no recourse under any explanation); \emph{\% 1-D}, \% of denied individuals who can overturn their prediction by changing 1 feature (i.e., who could have recourse from a feature-highlighting explanation); and \emph{\% n-D}, \% of denied individuals who can only overturn their prediction by changing 2 or more features simultaneously.}
\label{Table::DatasetModelMetrics}
\end{wraptable}

\begin{itemize}[itemsep=0.1em, topsep=0em,leftmargin=*]
    
    \item \emph{Feature Responsiveness} (\resp): We compute responsive scores from complete reachable sets that we enumerate using \cref{Alg::ReachableSetEnumeration}. %
    
    \item \emph{Standard Feature Attribution}: We consider model-agnostic methods that are widely used in the lending industry~\citep{FinRegLab2023}: \shap{}~\citep{lundberg2017unified}; and \lime{}~\citep{ribeiro2016should}.
    
    \item \emph{Actionable Feature Attribution}: We consider \emph{action-aware} variants of \shap{} and \lime{}: \shapaa{} and \limeaa{}. They aim to highlight responsive features by $\impscore{\xp}{j}{} \gets 0$ for immutable features.
 
\end{itemize}

\paragraph{On the Limits of Feature-Highlighting Explanations}

Our results  in \cref{Table::DatasetModelMetrics} highlight how current practices to comply with explainability mandates can help consumers achieve recourse. As shown, there is no case---i.e., for any model, any dataset, and any explanation method---where all individuals that receive feature-highlighting explanations could attain the target prediction through a single-feature intervention. Some require joint interventions. Others have no path to recourse.

On the \textds{heloc} dataset, for example, a lender who uses an \LR{} model would provide feature-highlighting explanations to 56.1\% of applicants. Among these individuals, 41.0\% could attain a desired prediction by changing single feature, 36.8\% could only do so by changing 2 or more features simultaneously, and the model assigns a fixed prediction to the remaining 22.2\%.

These results reflect the \emph{best} we can hope for when providing recourse with feature-highlighting explanations. Here, the 41.0\% of individuals who could achieve recourse by a single-feature intervention can only do so if we construct explanations with an \emph{ideal} method that assigns the highest scores to responsive features, and do not face additional actionability constraints. %

\sethlcolor{good}
\begin{table*}[t!]
\centering
    \centering
    \resizebox{\linewidth}{!}{

\begin{tabular}{ll*{5}r*{5}r}
\multicolumn{2}{c}{ } & \multicolumn{5}{c}{\LR{}} & \multicolumn{5}{c}{\XGB{}} \\
\cmidrule(l{3pt}r{3pt}){3-7} \cmidrule(l{3pt}r{3pt}){8-12}
\multicolumn{2}{c}{ } & \multicolumn{2}{c}{\regular{}} & \multicolumn{2}{c}{\filtered{}} & \multicolumn{1}{c}{ } & \multicolumn{2}{c}{\regular{}} & \multicolumn{2}{c}{\filtered{}} & \multicolumn{1}{c}{ } \\
\cmidrule(l{3pt}r{3pt}){3-4} \cmidrule(l{3pt}r{3pt}){5-6} \cmidrule(l{3pt}r{3pt}){8-9} \cmidrule(l{3pt}r{3pt}){10-11}

Dataset & 
Metrics & 
\lime{} & \shap{} & \limeaa{} & \shapaa{} & \resp & 
\lime{} & \shap{} & \limeaa{} & \shapaa{} & \resp\\

\cmidrule(lr){1-2} \cmidrule(lr){3-7} \cmidrule(lr){8-12} 

\ficoinfolong{} & \dmeMetrics{} & \cell{r}{100.0\%\\\color{bad}{92.7\%}\\7.3\%\\0.0\%\\4.0} & \cell{r}{100.0\%\\\color{bad}{77.3\%}\\22.7\%\\0.0\%\\4.0} & \cell{r}{100.0\%\\\color{bad}{76.8\%}\\23.2\%\\0.0\%\\4.0} & \cell{r}{100.0\%\\\color{bad}{70.3\%}\\29.7\%\\0.2\%\\4.0} & \cell{r}{41.0\%\\0.0\%\\100.0\%\\\cellcolor{good}{\textbf{100.0\%}}\\2.3} & \cell{r}{100.0\%\\\color{bad}{93.2\%}\\6.8\%\\0.0\%\\4.0} & \cell{r}{100.0\%\\\color{bad}{82.3\%}\\17.7\%\\0.0\%\\4.0} & \cell{r}{100.0\%\\\color{bad}{80.0\%}\\20.0\%\\0.0\%\\4.0} & \cell{r}{100.0\%\\\color{bad}{79.6\%}\\20.4\%\\0.0\%\\4.0} & \cell{r}{25.3\%\\0.0\%\\100.0\%\\\cellcolor{good}{\textbf{100.0\%}}\\2.5}\\

\cmidrule(lr){1-2} \cmidrule(lr){3-7} \cmidrule(lr){8-12}

\germaninfolong{} & \dmeMetrics{} & \cell{r}{100.0\%\\\color{bad}{91.7\%}\\8.3\%\\0.0\%\\4.0} & \cell{r}{100.0\%\\\color{bad}{100.0\%}\\0.0\%\\0.0\%\\4.0} & \cell{r}{100.0\%\\\color{bad}{59.4\%}\\40.6\%\\0.0\%\\4.0} & \cell{r}{100.0\%\\\color{bad}{65.1\%}\\34.9\%\\0.0\%\\4.0} & \cell{r}{74.2\%\\0.0\%\\100.0\%\\\cellcolor{good}{\textbf{100.0\%}}\\1.8} & \cell{r}{100.0\%\\\color{bad}{100.0\%}\\0.0\%\\0.0\%\\4.0} & \cell{r}{100.0\%\\\color{bad}{99.1\%}\\0.9\%\\0.0\%\\4.0} & \cell{r}{100.0\%\\\color{bad}{70.5\%}\\29.5\%\\0.0\%\\4.0} & \cell{r}{100.0\%\\\color{bad}{67.3\%}\\32.7\%\\0.0\%\\4.0} & \cell{r}{68.2\%\\0.0\%\\100.0\%\\\cellcolor{good}{\textbf{100.0\%}}\\1.8}\\

\cmidrule(lr){1-2} \cmidrule(lr){3-7} \cmidrule(lr){8-12}

\givemecreditinfolong{} & \dmeMetrics{} & \cell{r}{100.0\%\\\color{bad}{65.5\%}\\34.5\%\\0.0\%\\4.0} & \cell{r}{100.0\%\\\color{bad}{46.8\%}\\53.2\%\\0.0\%\\4.0} & \cell{r}{100.0\%\\\color{bad}{56.0\%}\\44.0\%\\0.0\%\\4.0} & \cell{r}{100.0\%\\\color{bad}{33.1\%}\\66.9\%\\22.8\%\\4.0} & \cell{r}{72.4\%\\0.0\%\\100.0\%\\\cellcolor{good}{\textbf{100.0\%}}\\2.4} & \cell{r}{100.0\%\\\color{bad}{41.8\%}\\58.2\%\\0.0\%\\4.0} & \cell{r}{100.0\%\\\color{bad}{43.3\%}\\56.7\%\\0.0\%\\4.0} & \cell{r}{100.0\%\\\color{bad}{31.6\%}\\68.4\%\\4.2\%\\4.0} & \cell{r}{100.0\%\\\color{bad}{30.6\%}\\69.4\%\\13.2\%\\4.0} & \cell{r}{76.0\%\\0.0\%\\100.0\%\\\cellcolor{good}{\textbf{100.0\%}}\\2.6}\\

\cmidrule(lr){1-2} \cmidrule(lr){3-7} \cmidrule(lr){8-12} 
\end{tabular}
}
    \caption{Responsiveness of feature-highlighting explanations for \LR{} and \XGB{} models for all methods and datasets. We defer results for \RF{} to \cref{Appendix::RFResults} for clarity. For each model, we generate explanations that highlight up to 4 top-scoring features from a given method. We report the proportion of individuals receiving an explanation (\emph{\% Presented with Explanations}) and the mean number of features in each explanation (\emph{\# Features Highlighted}). We also show the proportion of instances where all features are unresponsive (\emph{\% All Features Unresponsive}) highlighting {\color{bad}{positive values}}; at least one feature is responsive (\emph{\% At Least 1 Feature Responsive}), or all features are responsive (\emph{\% All Features Responsive}) highlighting the \textbf{\hl{best value}}.}
    \label{Table::MainMetrics}
\end{table*}

\paragraph{On Explanations with Feature Attribution Scores}

Our results show how standard methods for feature attribution can highlight features that are uninformative or misleading. Given the \LR{} model on the \textds{heloc} dataset, we find that 92.7\% and 77.3\% of explanations from \lime{} and \shap{} fail to highlight even one responsive feature. This stems from two issues:
\begin{itemize}[topsep=0pt, itemsep=0.1em]

\item \emph{Low Scores for Responsive Features}: Under the \LR{} model on the \textds{heloc} dataset, 41.0\% of denied individuals could be approved by altering a single feature. However, \lime{} and \shap{} do not highlight these features because they assign higher scores to other features (see \cref{Appendix::TopKPlots}). 

\item \emph{Fixed Predictions}: Under the \LR{} model on the \textds{heloc} dataset, 22.2\% of denied individuals cannot be approved under any intervention as they receive fixed prediction. These are instances where \lime{} and \shap{} (and their variants) can inflict harm by highlighting mutable features. For example, one individual who is assigned a fixed prediction would receive an explanation that highlights mutable features such as \textfn{AvgYearsInFile} and \textfn{NetFractionRevolvingBurden} under \shap{}, which gives the impression that intervening on them could lead to approval.

\end{itemize}

\paragraph{On Adapting Existing Methods}
Seeing how feature attribution methods like \lime{} and \shap{} can highlight features that are important but immutable, we study the potential to improve responsiveness using \emph{action-aware} variants \shapaa{} and \limeaa{}. Following a common belief that we can enforce actionability post-hoc~\citep[e.g.,][]{mothilal2020explaining}, we construct explanations using only actionable features. In \cref{Table::MainMetrics}, we see that \shapaa{} and \limeaa{} can highlight more responsive features. Given the \LR{} model in \textds{heloc}, for example, this strategy improves the proportion of explanations that contain at least one responsive feature by 7.0\% (i.e., $29.7\%$ of \shapaa{} vs. $22.7\%$ for \shap{}). One shortcoming of this approach is that we must filter features based on their actionability for all individuals, which may overlook features that is actionable for some individuals but not others.

\paragraph{On Explanations with Responsiveness Scores}
Our results show how practitioners can use our approach to comply with regulatory requirements and address the limitations of feature attribution methods.
When we construct feature-highlighting explanations using responsiveness scores, we present individuals with explanations that only contain responsive features (100\% on the \emph{\% All Features Responsive} metric across datasets and models in \cref{Table::MainMetrics}).
In contrast, only 0.2\% of \shapaa{} explanations of the \LR{} model in \textds{heloc} were fully responsive. For the remaining 99.8\%, each explanation contains at least one unresponsive feature that could lead individuals to intervene without achieving the target prediction.

Explanations based on responsiveness scores contain the \emph{most} responsive features that one can change independently to achieve recourse. In effect, we only provide explanations to individuals who can achieve recourse with a single-feature intervention. This may result in explanations that highlight fewer features on average. For example, \resp{} explanations for the \LR{} model on \textds{heloc} contained an average of 2.3 (out of 4) features. This behavior can mitigate harm as we avoid presenting explanations to individuals with fixed predictions (i.e., cannot change their predictions), or to individuals who could only do so with joint interventions.

\section{On the Limits of Feature Highlighting for Recourse}
\label{Sec::Demos}

\newtcbox{\hlrb}[1][]{
  on line,
  arc=2pt,
  colback=yellow!20,
  colframe=yellow!60!black,
  boxrule=0.1pt,
  left=0.01ex, right=0.01ex, top=0.01ex, bottom=0.01ex,
  enhanced, nobeforeafter, tcbox raise base,
  #1
}
\newcommand{\feat}[1]{{\small\textsf{{#1}}}}

Current mandates only require explanations to contain a list of most important features. In effect, most explanations lack details on \emph{how} individuals should intervene on them. Consider a person who receives an explanation that highlights \feat{income}. In this case, most consumers would assume that increasing \feat{income} would eventually lead to approval. In practice, this can yield counterproductive results---a responsive feature may be responsive in a way that is not monotonic (e.g., increase \feat{income} by at least \$1,000 but not more than \$2,500) or intuitive (e.g., decrease \feat{income} to be approved). As seen in this example, feature highlighting explanations can only provide recourse if features are: (1) responsive, (2) monotonic and (3) intuitive.

Building upon our results from \cref{Sec::Experiments}, \textbf{we show that even when methods like LIME and SHAP highlight responsive features, the necessary interventions to obtain recourse are not immediately apparent to consumers}. These results also highlight how we can use our machinery to check for more complicated notions of responsiveness (e.g., monotonicity).

\paragraph{Setup}
We use the same setup in \cref{Sec::Experiments} and fit an \XGB{} model for a version of the \textds{givemecredit} dataset $n = 23,459$ where we do not binarize continuous features. We construct feature highlighting explanations for each individual denied credit that contain up to four features based on scores from \shap{}, \lime{}, and \resp{}.
We estimate the responsiveness score (\resp{}) for each feature using a sample of $N = 500$ reachable points $\Rjset{\xp}{j}$ that we generate using \cref{Alg::ReachableSetSampling}. Our choice of $N = 500$ ensures that there is a 99\% chance that a feature that we claim is unresponsive has a true responsiveness $\leq 0.01$---i.e., at most of 1\% of actions could lead to recourse (see \cref{Appendix::SamplingDetails}).
In addition to measuring responsiveness of each feature under arbitrary changes, we use $\Rjset{\xp}{j}$ to evaluate whether responsiveness is monotonic or intuitive. 
We verify that a feature is intuitively responsive by finding at least one reachable point in the intuitive direction that attains the target prediction. 
We verify monotonicity of responsiveness by searching for a threshold value where all reachable points with feature $j$ less/greater than the threshold attain the target prediction. 
When we construct explanations with \resp{} through $\Rjset{\xp}{j}$, we only include features satisfying the necessary conditions for recourse depending on the additional information we assume is provided; if none exist, we don't not construct an explanation.

\paragraph{On the Need for Additional Information}

Our results reveal that omitting additional information on interventions undermines the value of feature-highlighting explanations for recourse. In \cref{tab::demoresults}, we see that there is no case where an individual can reliably achieve recourse using a feature-highlighting explanation based on \shap{} or \lime{}. In particular, 0\% of individuals receive explanations where all four feature are responsive, monotonic, and intuitive, and only 6\% of individuals receive an explanation where at least 1 feature obeys these conditions (\textbf{Features Only}).
This is either because interventions only lead to the target prediction when a feature takes on very specific values, or because the feature must be changed in counterintuitive ways.

\emph{Degree of Change}: Explanations can fail to provide a reliable path to recourse when they do not include information on \emph{how much} to change the value of responsive features. Consider credit utilization (\feat{CreditUtil}), which is in 32.3\% of explanations built using \shap{} scores. Our analysis reveals that it is responsive in 10.7\% of cases, but responsive \emph{and monotonic} in 9.0\% of cases. In other words, 1.7\% of individuals can only reliably intervene on this feature to attain a desired prediction when they have additional information on how much to change their credit utilization.
For instance, we point to an individual with $\feat{CreditUtil}=0.99$ who would be approved if they decrease the value of this feature to $\feat{CreditUtil} \in  (0.00, 0.50) \cup (0.65, 0.68)$. Without information on the degree of change, they may still be denied if they decrease their usage to the values of $\feat{CreditUtil} \cup (0.50, 0.65)$.

\emph{Direction of Change}: Explanations can also fail to provide a reliable path to recourse when they do not inlcude information on whether the decision subject should intervene on features by increasing or decreasing them. In this case, individuals who are shown responsive features may fail to obtain the target prediction because they must change features in a counterintuitive direction. Consider \feat{Income}, which is in 48.1\% of explanations built using \shap{} and is responsive in 20.6\% of all cases (i.e., ~43\% of explanations with \feat{Income}). In general, an individual who is denied credit and shown this feature would naturally assume that they can be approved by increasing this value. Yet, our analysis reveals this is only the case for 9.9\% of individuals; the remaining 10.7\% of individuals, where \feat{Income} is responsive, can only be approved by \emph{decreasing} \feat{Income}.

\newcommand{\includedistplot}[1]{\cell{c}{%
\includegraphics[trim=5px 12px 0px 5px,clip,height=0.065\textheight]{#1}}
}

\newcommand{\includedistplotnoaxis}[1]{\cell{c}{%
\includegraphics[trim=25px 12px 0px 10px,clip,height=0.08\textheight]{#1}}
}

\begin{table}[!t]
\centering
\scriptsize
\resizebox{0.95\textwidth}{!}{
\begin{tabular}{@{}>{\scriptsize}l@{\;}c@{}c@{}c@{}}
& \multicolumn{1}{c}{Best Case}
& \multicolumn{2}{c}{Current Practices}
\\
\cmidrule(lr){2-2} \cmidrule(lr){3-4}
\cell{l}{%
Information Provided to Consumers
} & 
\cell{c}{{\scriptsize\textsf{RESP}}} &
\cell{c}{{\scriptsize\textsf{LIME}}} &
\cell{c}{{\scriptsize\textsf{SHAP}}} \\
\midrule
\cell{l}{%
\textbf{Features Only}: \\
Features must be \emph{responsive, monotonic,}\\
\emph{intuitive} for decision subjects to obtain\\
recourse from explanations
} &
\includedistplot{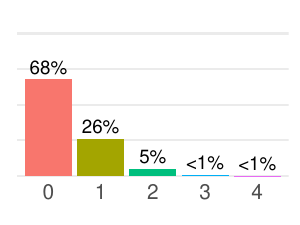} &
\includedistplot{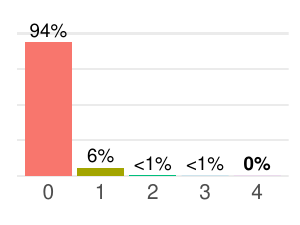} &
\includedistplot{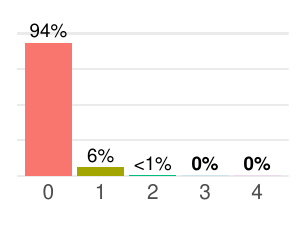} \\
\midrule
\cell{l}{%
\textbf{Features + Direction of Change}: \\
Features must be \emph{responsive, monotonic}\\
for decision subjects to obtain recourse\\
from explanations
} &
\includedistplot{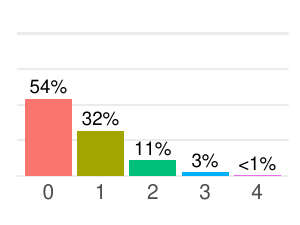} 
& \includedistplot{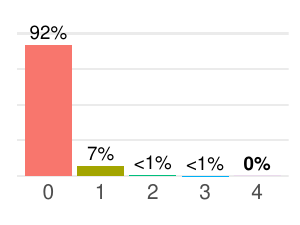} & 
\includedistplot{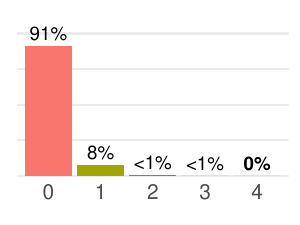} 
\\
\midrule
\cell{l}{%
\textbf{Features + Direction + Degree of Change}: \\
Features must be \emph{responsive} for\\
decision subjects to obtain recourse\\
from explanations
} &
\includedistplot{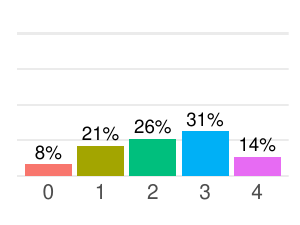} &
\includedistplot{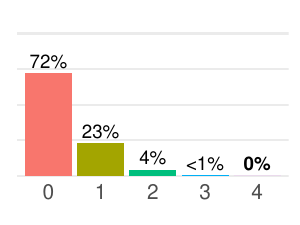} &
\includedistplot{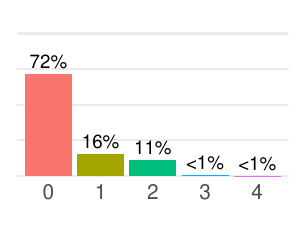} \\

\end{tabular}
}

\vspace{-1em}
\caption{Distribution of features that are responsive, monotonic and intuitive in feature-highlighting explanations using \lime{}, \shap{}, \resp{}. We plot the percentage of explanations given to consumers denied credit by the \XGB{} model in the \textds{\cp{givemecredit}} dataset with $k \in \{0, \ldots, 4\}$ features that are \emph{responsive} (can be changed to attain recourse), \emph{monotonic} (all changes below or beyond a threshold value will lead to recourse), and \emph{intuitive} (will lead to recourse if they are changed in a direction that aligns with common expectations).
We show that 94\% of consumers who receive feature-highlighting explanations built from \lime{} and \shap{} are unlikely overturn their predictions.
We can reduce this value to 72\% by including additional information on the degree and direction of change.
Under \resp{}, we only include features that meet the necessary conditions for recourse; if no features meet these conditions, we do not provide explanations for the consumer. Including additional information will decrease this proportion to 8\%.
}
\label{tab::demoresults}
\end{table}

\paragraph{Discussion}

When we construct explanations using methods like \shap{} and \lime{}, we cannot reliably tell when features are responsive, monotonic, or intuitive. In contrast, with \resp{}, we can verify whether each feature is monotonically or intuitively responsive through its reachable set $\Rjset{\xp}{j}$. If there are no features that meet these conditions, we would refrain from providing explanations. 
Hence, \resp{} represents ``optimal conditions'' for feature-highlighting explanations. However, we see that, without details on the magnitude or the direction of change, highlighting features alone is unlikely to provide recourse. Specifically, we cannot provide explanations to 68\% of denied individuals (\cref{tab::demoresults}). They would not benefit from feature-highlighting explanations because none of their features are responsive, monotonic and intuitive.
These results underscore the need to include additional information when explanations are meant to support recourse.

We could impose some of these conditions by enforcing constraints on how we train the model---i.e., we could ensure monotonicity by using a linear classifier like \LR{} rather than \XGB{}. Alternatively, we could use custom responsiveness scores to highlight features that meet all these conditions by inspecting their reachable sets $\Rjset{\xp}{j}$, as we have done in this section.
The \resp{} column in \cref{tab::demoresults} shows how effective custom responsiveness scores can be as a standalone solution. When there are no features that satisfy these conditions, we could highlight features that achieve weaker forms of responsiveness alongside additional information. In \cref{tab::demoresults}, we see that including additional information can reduce the proportion of cases we refrain from explaining to 8\%. This provides an alternative approach to ensure these conditions in a way that would not interfere with model development.

\section{Concluding Remarks}
\label{Sec::ConcludingRemarks}

Explanations are often seen as safeguards in consumer-facing applications as they can reveal information that can help exercise their broader rights pertaining to anti-discrimination, correction of erroneous information and recourse~\citep[][]{edwards2017slave}.

However, our findings suggest that current mandates may be insufficient in achieving their stated goals.
We showed that feature-highlighting explanations can fail to help consumers achieve a target prediction---providing \emph{reasons without recourse}. This arises because feature attribution methods like \shap{} and \lime{} overlook whether features are actionable and highlight immutable ones. They fail to capture counterfactual behavior and highlight unresponsive features. These explanations may also lead to harm by explaining fixed predictions. 
We further demonstrated that even when explanations highlight responsive features, feature-highlighting explanations may still fail to provide recourse if they omit information on how to change each feature. 

Evaluating whether mandates accomplish their goals is especially important as explanation mandates are often designed to achieve objectives that can be addressed using other techniques---e.g., anti-discrimination via auditing~\citep{saleiro2018aequitas,bellamy2019ai,skirzynski2025discrimination} or searching for less discriminatory models~\citep{black2024less,gillis2024operationalizing}.

\paragraph{Use Cases for Responsiveness Scores}
Our work has primarily focused on consumer finance applications because there are long-standing regulations on explanations and recourse in place. More broadly, we can draw on responsiveness scores to describe how models behave with respect to user interactions. 
\begin{itemize}
    \item \emph{Concept Annotations}: We can use responsiveness scores to select which concepts to confirm when performing test-time interventions under a constrained budget in concept-based models~\citep{koh2020concept,joren2024classification}. Interventions can prioritize confirming concepts with high responsiveness.
    \item \emph{Selective Feature Reporting}: Users can decide whether to provide optional features for personalized predictions based on responsiveness scores~\citep{joren2023participatory}; a high responsiveness score suggests that it is beneficial for the user to report the feature.
    \item \emph{Model Steering}: Users can identify features they can use to collectively steer model behavior with responsiveness scores~\citep{hardt2023algorithmic}. 
    \item \emph{Strategic Classification}: Model developers can preemptively identify features that users can manipulate to ``game'' their predictions and act upon them (e.g., make responsive features immutable on the platform)~\citep{hardt2016strategic}.
\end{itemize}

\paragraph{Limitations}

In applications like lending, actionability of features can be ambiguous; practitioners need to make assumptions on how features can change. In this work, we measured the responsiveness of features with respect to a conservative set of assumptions---indisputable constraints on how individuals can change their features. In this regime, our responsiveness scores can flag individuals with fixed predictions; but, we may not guarantee recourse as consumers may face additional constraints that we did not enforce. In practice, we can mitigate these issues by highlighting features that exceed a minimal level of responsiveness, or by eliciting constraints from decision subjects~\citep[see e.g.,][]{esfahani2024exploiting,de2022personalized,koh2024understanding}. 

\clearpage
\nottoggle{blind}{%
\subsection*{Acknowledgements}
This work is supported by the National Science Foundation under award IIS-2313105, and the NIH Bridge2AI Center Grant U54HG012510.%
}

\bibliographystyle{ext/iclr2025_conference}
\small\bibliography{responsiveness_scores}%

\clearpage
\appendix

\mtcaddpart[Supplementary Materials]

\thispagestyle{empty}
{%
\bfseries\Large\vspace{-2.0em}%
\begin{center}
Supplementary Materials
\end{center}%
\hrule
}
\renewcommand\ptctitle{}
\mtcsetfeature{parttoc}{open}{}%
\setlength{\ptcindent}{0pt}
\noptcrule
\parttoc[c]

\clearpage
\section{Supplementary Material for Section \ref{Sec::Computation}}
\label{Appendix::MIP}

\newcommand{\parti}[0]{\pi}
\newcommand{\jpart}[0]{\parti'}
\newcommand{\jpartnoj}[0]{\parti' \setminus \{j\}}

In this section, we provide additional implementation details for our methods to compute responsiveness scores in \cref{Sec::Computation}.

\paragraph{Partitions} Our implementation of algorithms in \cref{Sec::Computation} partitions the feature space into disjoint sets. Each part is made up of features that share joint constraints. More formally, we define a partition $\{\parti_1, \parti_2, \ldots, \parti_k \}$ of $[d]$ such that given two parts $\parti_m, \parti_n$, there are no joint constraints between all pairs $(p,q) \in \parti_m \times \parti_n$ of features. Another way to think about feature partitions would be as connected components in a graph, where features are nodes and edges represent joint constraints (i.e., $\exists \textrm{ edge } (p, q) \iff $ there are joint actionability constraints between $p$ and $q$).

In what follows, we denote $\jpart$ as a part that contains $j$ (i.e., $j \in \jpart$).

\subsection{Implementation Details for Reachable Set Enumeration}
\label{Appendix::FindMIP}

\paragraph{Description of {\normalfont\textsf{Find1DAction}} Routine}

The \textroutine{Find1DAction} routine in \cref{Alg::ReachableSetEnumeration} enumerates a set of possible actions from an intervention on single feature by recovering all possible solutions for an optimization problem of the form:
\begin{align*}
    \find{\xp, \runningaset} :=  \argmin_{\ab}  \|\ab\|_1 \ \st \ \ab \in \Ajset{\xp}{j}
\end{align*}
The routine takes as input:
\begin{itemize}
    \item $\xp \in \X$, a point
    \item $\Ajset{\xp}{j}$, the action set for feature $j$, representing both actionability constraints and feasible actions.
    \item $\Knogood{}$, a set of $\intrange{L}:=|\Knogood{}|$ actions enumerated over previous iterations
\end{itemize} 
Since actionability constraints specified in $\Ajset{\xp}{j}$ precisely define feasible actions, we overload notation to allow $\Ajset{\xp}{j}$ to represent the set of feasible actions.

At each iteration, it searches for the nearest single-feature action from the set $\ab \in \Ajset{\xp}{j}$ by solving a mixed-integer program formulation shown in \cref{MIP::FP}. The optimization procedure returns the nearest action when it exists, or returns a certificate of infeasibility, which indicates that there are no more actions to enumerate. As detailed in \cref{Alg::ReachableSetEnumeration}, if we find a solution $\ab^*$, we remove it from $\Ajset{\xp}{j}$ for the next iteration. In practice, we add solutions from each iteration to $\Knogood{}$ and add constraints with respect to each solution:

\begin{subequations}
\label{MIP::FP}
\footnotesize
\renewcommand{\arraystretch}{1.5}
\begin{equationarray}{@{}c@{}r@{\,}c@{\,}l>{\,}l>{\,}l@{\;}}
\min_{\ab} & \quad \sum_{k \in \jpart} &  & \akpos + \akneg &  & \notag \\[1.5em]
\st 
& a_j & \neq & 0 &  & \mipwhat{intervene on $j$} \label{Con::FP::MatchInter}\\
& a_k & = & \akpos - \akneg & k \in \jpart & \mipwhat{reconstruction of $a_k$} \label{Con::FP::Recon}\\
& \akpos & \geq & a_k & k \in \jpart & \mipwhat{positive component of $a_k$} \label{Con::FP::AbsValue1}\\
& \akneg & \geq & -a_k & k \in \jpart & \mipwhat{negative component of $a_k$} \label{Con::FP::AbsValue2} \\
& \akpos & \leq & \left|\sup_{\ab' \in \Ajset{\xp}{j}} a_k' \right|\sigma_k & k \in \jpart & \mipwhat{$\akpos > 0 \implies \sigma_k = 1$} \label{Con::FP::SepUB}\\
& \akneg & \leq & \left|\inf_{\ab' \in \Ajset{\xp}{j}} a_k' \right|(1-\sigma_k) & k \in \jpart & \mipwhat{$\akneg > 0 \implies \sigma_k = 0$} \label{Con::FP::SepLB}\\
& a_k  & = & a_{k,l} + \dpos{k}{l} - \dneg{k}{l} &  k \in \jpart, \aopt{l}\in \Knogood    & \mipwhat{maintain distance from prior actions} \label{Con::FP::ReachableSet1} \\
& \epsmin{} & \leq & \sum_{k \in \jpart} (\dpos{k}{l} +  \dneg{k}{l}) & \aopt{l}\in \Knogood  & \mipwhat{any solution is  $\epsmin$ away from $\aopt{l}$} \label{Con::FP::ReachableSet2}\\
& \dpos{k}{l} & \leq & M^{+}_{k,l} \chgind{k,l}   & k \in \jpart, l \in \intrange{L} & \mipwhat{$\dpos{k}{l} >0 \implies \chgind{k,l} = 1 $} \label{Con::FP::ReachableSet3}\\
& \dneg{k}{l} & \leq & M^{-}_{k,l}(1-\chgind{k,l})  & k \in \jpart, l \in \intrange{L} & \mipwhat{$\dneg{k}{l}>0 \implies \chgind{k,l} = 0$} \label{Con::FP::ReachableSet4}\\
& \ab & \in & \Ajset{\xp}{j} & & \mipwhat{joint actionability constraints on $j$} \label{Con::FP::SepActionSet}  \\
& \akpos, \akneg & \in & \R_{+} & k \in \jpart & \mipwhat{absolute value of $a_k$} \label{Con::FP::Abs1}  \\
& \dpos{k}{l},\dneg{k}{l} & \in & \R_{+} & k \in \jpart, l \in \intrange{L} & \mipwhat{signed distances from $a_{k,l}$} \label{Con::FP::SigneDelta}\\ 
& \chgind{k,l} & \in & \B & k \in \jpart, l \in \intrange{L}  & \mipwhat{sign indicator of $\delta_{k,l}$} \label{Con::FP::DeltaSign} \\
& \sigma_k & \in & \B & k \in \jpart  & \mipwhat{sign indicator of $a_k$} \label{Con::FP::ASign}
\end{equationarray} 
\end{subequations}

This formulation finds the closest feasible action that at $\epsmin{}$ away from each action in the set $\Ajset{\xp}{j}$. Here, the objective minimizes the $L_1$-norm of $a_k$ in terms of its positive and negative $\akpos - \akneg$, which are defined in constraints \eqref{Con::FP::Recon} and \eqref{Con::FP::Abs1}. 
The constraints enforce a minimum distance between $a_k$ and the $l^\textrm{th}$ solution from the set $\Knogood{}$ in terms of the distance variables $\dpos{k}{l}$ and $\dneg{k}{l}$, which are defined in constraints \eqref{Con::FP::ReachableSet1} and  \eqref{Con::FP::SigneDelta}.
Here, $\sigma_k:=\indic{a_k > 0}$ and $\chgind{k,l}:=\delta_{k,l}$ are binary variables set to 1 when $a_k$ and $\delta_{k,l}$ have positive signs, respectively. The formulation ensures these variables to ensure that signed components can have a positive value through constraints \eqref{Con::FP::ASign} and  \eqref{Con::FP::DeltaSign}, respectively.

The first constraint, \eqref{Con::FP::MatchInter} enforces that we intervene on feature $j$. The remaining constraints describe three key requirements for $\ab$:
\begin{enumerate}
    \item Sufficient distance from prior solutions (constraint \eqref{Con::FP::ReachableSet2})
    \item Adherence to separable actionability constraints (constraint \eqref{Con::FP::SepUB}, \eqref{Con::FP::SepLB}, \eqref{Con::FP::ReachableSet3}, \eqref{Con::FP::ReachableSet4})
    \item Adherence to joint actionability constraints (constraint \eqref{Con::FP::SepActionSet})
\end{enumerate}

Constraint \eqref{Con::FP::ReachableSet2} ensures that given $\epsmin{} > 0$, $\|\ab - \ab_l\|_1 \geq \epsmin{} \; \forall \, \ab_l \in \Knogood$. We set $\epsmin{} = 0.5$ for our experiments with discrete datasets.

Constraints \eqref{Con::FP::SepUB}, \eqref{Con::FP::SepLB} ensure that $a_k$ is feasible under separable constraints on $k \in \jpart$ and that only one of $\akpos$ or $\akneg$ is strictly positive. Similarly, constraints \eqref{Con::FP::ReachableSet3}, \eqref{Con::FP::ReachableSet4} ensure that the distances between $\ab$ and each $\ab_l$ are within some bound. We achieve this by setting ``Big-M'' parameters $M^{+}_{k,l}, M^{-}_{k,l}$, which represent the upper bound for $\dpos{k}{l}$ and $\dneg{k}{l}$. For each feature $k \in \jpart$, we let
\[
    M^{+}_{k,l} := \left|\sup_{\ab' \in \Ajset{\xp}{j}} a_k' - a_{k,l}\right|, \
    M^{-}_{k,l} := \left|\inf_{\ab' \in \Ajset{\xp}{j}} a_k' - a_{k,l}\right|, \
\]
Along with the indicator variable $\chgind{k,l}$, $M^{+}_{k,l}, M^{-}_{k,l}$ ensure that only one of $\dpos{k}{l}$ or $\dneg{k}{l}$ is strictly positive and is feasible under separable actionability constraints. 

Constraint \eqref{Con::FP::SepActionSet} ensures that $\ab$ also adheres to joint actionability constraints. These constraints will exist if and only if $|\jpart| > 1$. See \citep{kothari2024prediction} for examples of how we can explicitly encode joint actionability constraints into \cref{MIP::FP}.

This formulation is adapts the MIP in \citep{kothari2024prediction} for a task where we only need to enumerate actions with respect to a single-feature intervention $\vb$.

\subsection{Implementation Details for Reachable Set Sampling}
\label{Appendix::SamplingDetails}

The sampling algorithm (\cref{Alg::ReachableSetSampling}) requires additional considerations---most notably the sample size $N$.

\paragraph{Choosing a Sample Size}

The sample size $\sampsize$ controls the precision of the estimated responsiveness score $\rscorehat{\xp}{j}$. We formalize precision using confidence intervals by treating $\rscorehat{\xp}{j}$ as a binomial distribution parameter:
\begin{remark}
\label{Def::EstimatedRScore}
    Given a point $\xp \in \X$, let $\Rjhat{\xp}{j}$ denote a sample of $\sampsize{}$ points drawn uniformly from the reachable set $\Rjset{\xp}{j}$. Given any model $\clf{}: \X \to \Y$, we can estimate the responsiveness score for feature $j$ as 
    $\rscorehat{\xp}{j} := \tfrac{1}{\sampsize} \sum_{\xq \in \Rjhat{\xp}{j}} \indic{\clf{\xq} = \ytarget}.$
    Given a significance level $\alpha \in (0,1)$, we have that:
    \begin{align*}
        \Pr{\rscore{\xp}{j} \in [\rscoretilde{\xp}{j} - \mathcal{E}, \rscoretilde{\xp}{j} + \mathcal{E}] } \geq 1 - \alpha
    \end{align*}
    Here: $\mathcal{E} := \kappa \sqrt{\frac{1}{\sampsize + \kappa^2} \rscoretilde{\xp}{j}(1-\rscoretilde{\xp}{j})}$ and $\rscoretilde{\xp}{j} := \frac{1}{\sampsize + \kappa^2} \left( S + \frac{\kappa^2}{2}\right)$ is a corrected estimator to improve coverage when $\rscore{\xp}{j} \in \{0, 1\}$~\citep[][]{brown2001interval}, $\success \defeq{} |\{\xq \in \Rjhat{\xp}{j} \mid \clf{\xq} = \ytarget\}|$ is the subset of responsive points, and $\kappa \defeq{} \Phi^{-1}(1-\tfrac{\alpha}{2})$ is a constant based on the Normal CDF $\Phi(\cdot).$
\end{remark}

The Agresti–Coull interval above is an approximate confidence interval for a binomial proportion~\citep{agresti1998approximate}, offering an improvement over the standard normal approximation known as the Wald Interval. It is particularly effective for small proportion values, providing more reliable coverage---the probability that the interval contains the true parameter value~\citep{brown2001interval}.

\begin{fact}
\label{Remark::CalcAB}
Given $\alpha$ and $\sampsize$, $\mathcal{E}$ is maximized when $S = \frac{N}{2}$ and attains its minimum at $S = 0$ and $S = N$.
\end{fact}
\begin{proof}
    Let $z = \frac{S + \frac{\kappa^2}{2}}{N + \kappa^2}$. Then, we have:
    \[
    \mathcal{E} = \kappa\sqrt{\frac{z(1-z)}{(\sampsize+\kappa^2)}}
    \]
    Since $\sampsize$ and $\alpha$ are fixed ($\kappa$ is defined by $\alpha$), $\mathcal{E}$ can be expressed as a function of $z$ of the form \(h(z) = c\sqrt{z(1-z)}\) where $c \in \R^+$. We observe that h(z) is a concave function whose first derivative can be expressed as:
    \begin{align*}
        h'(z) = %
        \frac{c(1-2z)}{2\sqrt{z(1-z)}} \qquad 
         h''(z) = -\frac{c}{4}[z(1-z)]^{-\frac{3}{2}}
    \end{align*}
    Since $z > 0$ and $c > 0$, we can see that $h(z)$ attains a maximum value at $z' = \frac{1}{2}$ since $h(z') = 0$ and $h''(z) < 0$. 

    We see that $h'(z) > 0$ where $z < \frac{1}{2}$, meaning it is increasing for $z \in (0,\frac{1}{2}]$. Thus, the local minimum is achieved at the smallest possible $z$---when $S = 0$.

    Similarly, for $z \in [\frac{1}{2}, 1)$, $h'(z) < 0$ and the local minimum is achieved at the largest possible $z$---when $S = N$.

    Note that the value of $h$ (or $\mathcal{E}$) are the same at those two points.
    
\end{proof}

Using the \cref{Remark::CalcAB}, we can a sample size $\sampsize$ in terms of $\alpha$ in following ways.
\begin{enumerate}
    \item Control the precision when $S = 0$ (i.e., no points in $\Rjhat{\xp}{j}$ are responsive) $\iff$ control the width of the shortest interval
    \item Control the precision when $S = \tfrac{\sampsize}{2}$ (i.e., half of the points in $\Rjhat{\xp}{j}$ are responsive) $\iff$ control the width of the widest interval
\end{enumerate}
Either way, we fix $\alpha$ and solve for $\sampsize$ given the width of the interval $\mathcal{E}$ at a specified $S$. Below we provide a table of the smallest $\sampsize$ needed for different $\mathcal{E}$---interval widths---at common values of $\alpha$ for the two methods:

\begin{table}[htbp]
  \centering
  \begin{minipage}{0.48\textwidth}
    \centering
\begin{tabular}{lrrrr}
\multicolumn{1}{c}{ } & \multicolumn{4}{c}{Width of Interval ($\mathcal{E}$)} \\
\cmidrule(l{3pt}r{3pt}){2-5}
$\alpha$ & 0.01 & 0.02 & 0.05 & 0.10 \\ 
\midrule
0.01 & 461 & 227 & 86 & 39 \\
0.05 & 267 & 132 & 50 & 23 \\
0.10 & 188 & 93 & 35 & 16 \\
\bottomrule
\end{tabular}
    \caption{Minimum $N$ required to ensure the shortest confidence interval is less than $2\mathcal{E}$ (Method 1)}
    \label{Tab::SampSize1}
  \end{minipage}\hfill
  \begin{minipage}{0.48\textwidth}
    \centering
\begin{tabular}{lrrrr}
\multicolumn{1}{c}{ } & \multicolumn{4}{c}{Width of Interval ($\mathcal{E}$)} \\
\cmidrule(l{3pt}r{3pt}){2-5}
$\alpha$ & 0.01 & 0.02 & 0.05 & 0.10 \\ 
\midrule
0.01 & 16581 & 4141 & 657 & 160 \\
0.05 & 9600 & 2398 & 381 & 93 \\
0.10 & 6762 & 1689 & 268 & 65\\
\bottomrule
\end{tabular}
    \caption{Minimum $N$ required to ensure the widest confidence interval is less than $2\mathcal{E}$ (Method 2)}
    \label{Tab::SampSize2}
  \end{minipage}
\end{table}

\paragraph{Description of the {\normalfont\textsf{Sample1DAction}} Routine}
Let $j$ be the feature that we are intervening on. 

\textbf{Case 1:} $|\pi'| = 1$ (i.e., $\pi' = \{j\}$, $j$ is not jointly constrained with other features). 

Here, there are no downstream effects from intervening on $j$. We take a uniformly random intervention from $\Vjset{\xp}{j}$: 
\[
\ab^* \sim \Ajset{\xp}{j}
\]
which abides by $j$'s separable actionability constraints like feature bounds and monotonicity.

\textbf{Case 2:} $|\pi'| > 1$ (i.e., $j$ is jointly constrained with other features)

\newcommand{\discpart}[0]{\pi'_{\textrm{disc}}}
\newcommand{\ctspart}[0]{\pi'_{\textrm{cts}}}
We breakdown the partition $\pi'$ into three disjoint subsets:
\[
\pi' = \{j\} \cup \discpart \cup \ctspart
\]
where $\discpart$ and $\ctspart$ are the sets of discrete and and continuous features in $\pi'$ respectively.

We consider the following three sub-cases:

\emph{Case 2a:} $|\ctspart| = 0$---all features in $\pi'$ are discrete.

We repeatedly solve the MIP in \textsf{Find1DAction} and take a sample from the resulting set of feasible actions.
        
\emph{Case 2b:} $|\discpart| = 0$---all features in $\pi'$ are continuous.

We sample action values that abide by separable actionability constraints for each feature in $\discpart$.

\emph{Case 2c:} $|\ctspart|, |\discpart| > 0 $---part contains discrete and continuous features.

We run the sampling steps in \emph{Case 2a, 2b} for $\discpart \cup \{j\}$ and $\ctspart$ to get $\ab_{\textrm{disc}}$ and $\ab_{\textrm{cts}}$.

We then check feasibility on $\ab^* = \ab_{\textrm{disc}} + \ab_{\textrm{cts}}$ by running $\textsf{CheckFeasibility}(\ab^*, \Ajset{\xp}{j})$.

\paragraph{Description of {\normalfont \textsf{CheckFeasibility}} Routine}

We describe the implementation for the $\feasible{\xb, \ab^*, \runningaset}$ in \cref{Alg::ReachableSetSampling}. Contrary to the MIP formulation in \cref{Appendix::FindMIP}, given the original point $\xp \in \X$ and the sampled action $\ab^*$, we solve the MIP once.

\begin{subequations}
\label{MIP::CF}
\footnotesize
\renewcommand{\arraystretch}{1.5}
\begin{equationarray}{@{}c@{}r@{\,}c@{\,}l>{\,}l>{\,}l@{\;}}
\min_{\ab} & \quad 1 &  &  &  & \notag \\[1.5em]
\st 
& \ab & = & \ab^* &  & \mipwhat{match action $\ab^*$} \label{Con::CF::MatchInter}\\
& a_k & = & \akpos - \akneg & k \in \jpart & \mipwhat{reconstruction of $a_k$} \label{Con::CF::Recon}\\
& \akpos & \geq & a_k & k \in \jpart & \mipwhat{positive component of $a_k$} \label{Con::CF::AbsValue1}\\
& \akneg & \geq & -a_k & k \in \jpart & \mipwhat{negative component of $a_k$} \label{Con::CF::AbsValue2} \\
& \akpos & \leq & \left|\sup_{\ab' \in \Ajset{\xp}{j}} a_k' \right|\sigma_k & k \in \jpart & \mipwhat{$\akpos > 0 \implies \sigma_k = 1$} \label{Con::CF::SepUB}\\
& \akneg & \leq & \left|\inf_{\ab' \in \Ajset{\xp}{j}} a_k' \right|(1-\sigma_k) & k \in \jpart & \mipwhat{$\akneg > 0 \implies \sigma_k = 0$} \label{Con::CF::SepLB}\\
& \ab & \in & \Ajset{\xp}{j} & & \mipwhat{joint actionability constraints on $j$} \label{Con::CF::SepActionSet}  \\
& \akpos, \akneg & \in & \R_{+} & k \in \jpart & \mipwhat{absolute value of $a_k$} \label{Con::CF::Abs1}  \\
& \sigma_k & \in & \B & k \in \jpart  & \mipwhat{sign indicator of $a_k$} \label{Con::CF::ASign}
\end{equationarray} 
\end{subequations}

The formulation is a variant of the problem in \cref{Appendix::FindMIP}, where:
\begin{itemize}
    \item $\ab = \ab^*$,
    \item $\Knogood = \varnothing$,
    \item and set the objective to $\min_{\ab} \; 1$
\end{itemize}

Hence $\feasible{\xb, \ab^*, \runningaset} = 1$ if $\ab^*$ is feasible under actionability constraints and 0 otherwise.

In practice, we run \textsf{Sample1DAction} and \textsf{CheckFeasibility} in \cref{Alg::ReachableSetSampling} in batches for efficiency; rather than sampling one action and checking feasibility, we sample $\tilde{N} >> N$ points and then check feasibility at once. We sample more than the required $N$ points to account for rejected samples in the \textsf{CheckFeasibility} step.

\clearpage
\section{Supplementary Experiment Details}
\label{Appendix::ExperimentDetails}

\subsection{Details for the {\normalfont\textds{heloc}} Dataset}

The \textds{heloc} dataset was created to predict repayment on Home Equity Line of Credit HELOC applications; these are loans that use people's homes as collateral. The anonymized version of the dataset was developed by FICO for use in an Explainable Machine Learning Challenge in 2018~\citep{fico}.
Each instance in the dataset is an application for a home equity loan from a single applicant. There are $n = 10,459$ instances and $d = 23$ features. Here, the label $y_i = 0$ if an applicant has been more than 90 days overdue on their payments in the last 2 years and $y_i = 1$ otherwise. We thermometer encode continuous or integer features after dropping rows and features with missing data (see \cref{Tab::ficoASet}). See \repourl{} for dataset processing code.

\paragraph{Actionability Constraints}
\begin{table*}[!h]
\centering
\resizebox{0.7\linewidth}{!}{\begin{tabular}{lllllllr}
\toprule
\textheader{Name} & \textheader{Type} & \textheader{LB} & \textheader{UB} & \textheader{Actionability} & \textheader{Sign} & \textheader{Joint Constraints} & \textheader{Partition ID} \\
\midrule
\textfn{NumInstallTrades$\geq$2} & $\{0,1\}$ & 0 & 1 & Yes & $+$ & 20, 21, 24, 25, 28, 29, 32, 33 & 14 \\
\textfn{NumInstallTradesWBalance$\geq$2} & $\{0,1\}$ & 0 & 1 & Yes & $+$ & 20, 21, 24, 25, 28, 29, 32, 33 & 14 \\
\textfn{NumInstallTrades$\geq$3} & $\{0,1\}$ & 0 & 1 & Yes & $+$ & 20, 21, 24, 25, 28, 29, 32, 33 & 14 \\
\textfn{NumInstallTradesWBalance$\geq$3} & $\{0,1\}$ & 0 & 1 & Yes & $+$ & 20, 21, 24, 25, 28, 29, 32, 33 & 14 \\
\textfn{NumInstallTrades$\geq$5} & $\{0,1\}$ & 0 & 1 & Yes & $+$ & 20, 21, 24, 25, 28, 29, 32, 33 & 14 \\
\textfn{NumInstallTradesWBalance$\geq$5} & $\{0,1\}$ & 0 & 1 & Yes & $+$ & 20, 21, 24, 25, 28, 29, 32, 33 & 14 \\
\textfn{NumInstallTrades$\geq$7} & $\{0,1\}$ & 0 & 1 & Yes & $+$ & 20, 21, 24, 25, 28, 29, 32, 33 & 14 \\
\textfn{NumInstallTradesWBalance$\geq$7} & $\{0,1\}$ & 0 & 1 & Yes & $+$ & 20, 21, 24, 25, 28, 29, 32, 33 & 14 \\
\textfn{NumRevolvingTrades$\geq$2} & $\{0,1\}$ & 0 & 1 & Yes & $+$ & 22, 23, 26, 27, 30, 31, 34, 35 & 15 \\
\textfn{NumRevolvingTradesWBalance$\geq$2} & $\{0,1\}$ & 0 & 1 & Yes & $+$ & 22, 23, 26, 27, 30, 31, 34, 35 & 15 \\
\textfn{NumRevolvingTrades$\geq$3} & $\{0,1\}$ & 0 & 1 & Yes & $+$ & 22, 23, 26, 27, 30, 31, 34, 35 & 15 \\
\textfn{NumRevolvingTradesWBalance$\geq$3} & $\{0,1\}$ & 0 & 1 & Yes & $+$ & 22, 23, 26, 27, 30, 31, 34, 35 & 15 \\
\textfn{NumRevolvingTrades$\geq$5} & $\{0,1\}$ & 0 & 1 & Yes & $+$ & 22, 23, 26, 27, 30, 31, 34, 35 & 15 \\
\textfn{NumRevolvingTradesWBalance$\geq$5} & $\{0,1\}$ & 0 & 1 & Yes & $+$ & 22, 23, 26, 27, 30, 31, 34, 35 & 15 \\
\textfn{NumRevolvingTrades$\geq$7} & $\{0,1\}$ & 0 & 1 & Yes & $+$ & 22, 23, 26, 27, 30, 31, 34, 35 & 15 \\
\textfn{NumRevolvingTradesWBalance$\geq$7} & $\{0,1\}$ & 0 & 1 & Yes & $+$ & 22, 23, 26, 27, 30, 31, 34, 35 & 15 \\
\textfn{YearsOfAccountHistory} & $\mathbb{Z}$ & 0 & 50 & No &  & 5, 17, 18, 19 & 5 \\
\textfn{YearsSinceLastDelqTrade$\leq$1} & $\{0,1\}$ & 0 & 1 & Yes & $+$ & 5, 17, 18, 19 & 5 \\
\textfn{YearsSinceLastDelqTrade$\leq$3} & $\{0,1\}$ & 0 & 1 & Yes & $+$ & 5, 17, 18, 19 & 5 \\
\textfn{YearsSinceLastDelqTrade$\leq$5} & $\{0,1\}$ & 0 & 1 & Yes & $+$ & 5, 17, 18, 19 & 5 \\
\textfn{NetFractionInstallBurden$\geq$10} & $\{0,1\}$ & 0 & 1 & Yes & $+$ & 36, 37, 38 & 16 \\
\textfn{NetFractionInstallBurden$\geq$20} & $\{0,1\}$ & 0 & 1 & Yes & $+$ & 36, 37, 38 & 16 \\
\textfn{NetFractionInstallBurden$\geq$50} & $\{0,1\}$ & 0 & 1 & Yes & $+$ & 36, 37, 38 & 16 \\
\textfn{NetFractionRevolvingBurden$\geq$10} & $\{0,1\}$ & 0 & 1 & Yes & $+$ & 39, 40, 41 & 17 \\
\textfn{NetFractionRevolvingBurden$\geq$20} & $\{0,1\}$ & 0 & 1 & Yes & $+$ & 39, 40, 41 & 17 \\
\textfn{NetFractionRevolvingBurden$\geq$50} & $\{0,1\}$ & 0 & 1 & Yes & $+$ & 39, 40, 41 & 17 \\
\textfn{AvgYearsInFile$\geq$3} & $\{0,1\}$ & 0 & 1 & Yes & $+$ & 6, 7, 8 & 6 \\
\textfn{AvgYearsInFile$\geq$5} & $\{0,1\}$ & 0 & 1 & Yes & $+$ & 6, 7, 8 & 6 \\
\textfn{AvgYearsInFile$\geq$7} & $\{0,1\}$ & 0 & 1 & Yes & $+$ & 6, 7, 8 & 6 \\
\textfn{MostRecentTradeWithinLastYear} & $\{0,1\}$ & 0 & 1 & Yes &  & 9, 10 & 7 \\
\textfn{MostRecentTradeWithinLast2Years} & $\{0,1\}$ & 0 & 1 & Yes &  & 9, 10 & 7 \\
\textfn{ExternalRiskEstimate$\geq$40} & $\{0,1\}$ & 0 & 1 & No &  & -- & 0 \\
\textfn{ExternalRiskEstimate$\geq$50} & $\{0,1\}$ & 0 & 1 & No &  & -- & 1 \\
\textfn{ExternalRiskEstimate$\geq$60} & $\{0,1\}$ & 0 & 1 & No &  & -- & 2 \\
\textfn{ExternalRiskEstimate$\geq$70} & $\{0,1\}$ & 0 & 1 & No &  & -- & 3 \\
\textfn{ExternalRiskEstimate$\geq$80} & $\{0,1\}$ & 0 & 1 & No &  & -- & 4 \\
\textfn{AnyDerogatoryComment} & $\{0,1\}$ & 0 & 1 & No &  & -- & 8 \\
\textfn{AnyTrade120DaysDelq} & $\{0,1\}$ & 0 & 1 & No &  & -- & 9 \\
\textfn{AnyTrade90DaysDelq} & $\{0,1\}$ & 0 & 1 & No &  & -- & 10 \\
\textfn{AnyTrade60DaysDelq} & $\{0,1\}$ & 0 & 1 & No &  & -- & 11 \\
\textfn{AnyTrade30DaysDelq} & $\{0,1\}$ & 0 & 1 & No &  & -- & 12 \\
\textfn{NoDelqEver} & $\{0,1\}$ & 0 & 1 & No &  & -- & 13 \\
\textfn{NumBank2NatlTradesWHighUtilizationGeq2} & $\{0,1\}$ & 0 & 1 & Yes & $+$ & -- & 18 \\
\bottomrule
\end{tabular}

}
\caption{
Separable Actionability Constraints for the processed \textds{heloc} dataset. 
\textbf{Type} indicates the feature type ($\Z$ for integer, $\{0,1\}$ for binary). \textbf{LB}, \textbf{UB} are the lower and upper bounds for the feature.
\textbf{Actionability} indicates whether the feature is globally actionable.
\textbf{Sign} indicates if the feature can only increase (+) or decrease (-).
\textbf{Joint Constraints} are a list non-separable constraint indices it is tied to (if any).
\textbf{Partition ID} indicates which partition the feature belongs to.
}
\label{Tab::ficoASet}
\end{table*}

The joint actionability constraints listed in \cref{Tab::ficoASet} include:
\begin{constraints}
\item DirectionalLinkage: Actions on \textfn{NumRevolvingTradesWBalance$\geq$2} will induce to actions on ['\textfn{NumRevolvingTrades$\geq$2}'].Each unit change in \textfn{NumRevolvingTradesWBalance$\geq$2} leads to:1.00-unit change in \textfn{NumRevolvingTrades$\geq$2}
\item DirectionalLinkage: Actions on \textfn{NumInstallTradesWBalance$\geq$2} will induce to actions on ['\textfn{NumInstallTrades$\geq$2}'].Each unit change in \textfn{NumInstallTradesWBalance$\geq$2} leads to:1.00-unit change in \textfn{NumInstallTrades$\geq$2}
\item DirectionalLinkage: Actions on \textfn{NumRevolvingTradesWBalance$\geq$3} will induce to actions on ['\textfn{NumRevolvingTrades$\geq$3}'].Each unit change in \textfn{NumRevolvingTradesWBalance$\geq$3} leads to:1.00-unit change in \textfn{NumRevolvingTrades$\geq$3}
\item DirectionalLinkage: Actions on \textfn{NumInstallTradesWBalance$\geq$3} will induce to actions on ['\textfn{NumInstallTrades$\geq$3}'].Each unit change in \textfn{NumInstallTradesWBalance$\geq$3} leads to:1.00-unit change in \textfn{NumInstallTrades$\geq$3}
\item DirectionalLinkage: Actions on \textfn{NumRevolvingTradesWBalance$\geq$5} will induce to actions on ['\textfn{NumRevolvingTrades$\geq$5}'].Each unit change in \textfn{NumRevolvingTradesWBalance$\geq$5} leads to:1.00-unit change in \textfn{NumRevolvingTrades$\geq$5}
\item DirectionalLinkage: Actions on \textfn{NumInstallTradesWBalance$\geq$5} will induce to actions on ['\textfn{NumInstallTrades$\geq$5}'].Each unit change in \textfn{NumInstallTradesWBalance$\geq$5} leads to:1.00-unit change in \textfn{NumInstallTrades$\geq$5}
\item DirectionalLinkage: Actions on \textfn{NumRevolvingTradesWBalance$\geq$7} will induce to actions on ['\textfn{NumRevolvingTrades$\geq$7}'].Each unit change in \textfn{NumRevolvingTradesWBalance$\geq$7} leads to:1.00-unit change in \textfn{NumRevolvingTrades$\geq$7}
\item DirectionalLinkage: Actions on \textfn{NumInstallTradesWBalance$\geq$7} will induce to actions on ['\textfn{NumInstallTrades$\geq$7}'].Each unit change in \textfn{NumInstallTradesWBalance$\geq$7} leads to:1.00-unit change in \textfn{NumInstallTrades$\geq$7}
\item DirectionalLinkage: Actions on \textfn{YearsSinceLastDelqTrade$\leq$1} will induce to actions on ['\textfn{YearsOfAccountHistory}'].Each unit change in \textfn{YearsSinceLastDelqTrade$\leq$1} leads to:-1.00-unit change in \textfn{YearsOfAccountHistory}
\item DirectionalLinkage: Actions on \textfn{YearsSinceLastDelqTrade$\leq$3} will induce to actions on ['\textfn{YearsOfAccountHistory}'].Each unit change in \textfn{YearsSinceLastDelqTrade$\leq$3} leads to:-3.00-unit change in \textfn{YearsOfAccountHistory}
\item DirectionalLinkage: Actions on \textfn{YearsSinceLastDelqTrade$\leq$5} will induce to actions on ['\textfn{YearsOfAccountHistory}'].Each unit change in \textfn{YearsSinceLastDelqTrade$\leq$5} leads to:-5.00-unit change in \textfn{YearsOfAccountHistory}
\item ReachabilityConstraint: The values of [\textfn{MostRecentTradeWithinLastYear}, \textfn{MostRecentTradeWithinLast2Years}] must belong to one of 4 values with custom reachability conditions.
\item ThermometerEncoding: Actions on [\textfn{YearsSinceLastDelqTrade$\leq$1}, \textfn{YearsSinceLastDelqTrade$\leq$3}, \textfn{YearsSinceLastDelqTrade$\leq$5}] must preserve thermometer encoding of YearsSinceLastDelqTradeleq., which can only decrease.Actions can only turn off higher-level dummies that are on, where \textfn{YearsSinceLastDelqTrade$\leq$1} is the lowest-level dummy and \textfn{YearsSinceLastDelqTrade$\leq$5} is the highest-level-dummy.
\item ThermometerEncoding: Actions on [\textfn{AvgYearsInFile$\geq$3}, \textfn{AvgYearsInFile$\geq$5}, \textfn{AvgYearsInFile$\geq$7}] must preserve thermometer encoding of AvgYearsInFilegeq., which can only increase.Actions can only turn on higher-level dummies that are off, where \textfn{AvgYearsInFile$\geq$3} is the lowest-level dummy and \textfn{AvgYearsInFile$\geq$7} is the highest-level-dummy.
\item ThermometerEncoding: Actions on [\textfn{NetFractionRevolvingBurden$\geq$10}, \textfn{NetFractionRevolvingBurden$\geq$20}, \textfn{NetFractionRevolvingBurden$\geq$50}] must preserve thermometer encoding of NetFractionRevolvingBurdengeq., which can only decrease.Actions can only turn off higher-level dummies that are on, where \textfn{NetFractionRevolvingBurden$\geq$10} is the lowest-level dummy and \textfn{NetFractionRevolvingBurden$\geq$50} is the highest-level-dummy.
\item ThermometerEncoding: Actions on [\textfn{NetFractionInstallBurden$\geq$10}, \textfn{NetFractionInstallBurden$\geq$20}, \textfn{NetFractionInstallBurden$\geq$50}] must preserve thermometer encoding of NetFractionInstallBurdengeq., which can only decrease.Actions can only turn off higher-level dummies that are on, where \textfn{NetFractionInstallBurden$\geq$10} is the lowest-level dummy and \textfn{NetFractionInstallBurden$\geq$50} is the highest-level-dummy.
\item ThermometerEncoding: Actions on [\textfn{NumRevolvingTradesWBalance$\geq$2}, \textfn{NumRevolvingTradesWBalance$\geq$3}, \textfn{NumRevolvingTradesWBalance$\geq$5}, \textfn{NumRevolvingTradesWBalance$\geq$7}] must preserve thermometer encoding of NumRevolvingTradesWBalancegeq., which can only decrease.Actions can only turn off higher-level dummies that are on, where \textfn{NumRevolvingTradesWBalance$\geq$2} is the lowest-level dummy and \textfn{NumRevolvingTradesWBalance$\geq$7} is the highest-level-dummy.
\item ThermometerEncoding: Actions on [\textfn{NumRevolvingTrades$\geq$2}, \textfn{NumRevolvingTrades$\geq$3}, \textfn{NumRevolvingTrades$\geq$5}, \textfn{NumRevolvingTrades$\geq$7}] must preserve thermometer encoding of NumRevolvingTradesgeq., which can only decrease.Actions can only turn off higher-level dummies that are on, where \textfn{NumRevolvingTrades$\geq$2} is the lowest-level dummy and \textfn{NumRevolvingTrades$\geq$7} is the highest-level-dummy.
\item ThermometerEncoding: Actions on [\textfn{NumInstallTradesWBalance$\geq$2}, \textfn{NumInstallTradesWBalance$\geq$3}, \textfn{NumInstallTradesWBalance$\geq$5}, \textfn{NumInstallTradesWBalance$\geq$7}] must preserve thermometer encoding of NumInstallTradesWBalancegeq., which can only decrease.Actions can only turn off higher-level dummies that are on, where \textfn{NumInstallTradesWBalance$\geq$2} is the lowest-level dummy and \textfn{NumInstallTradesWBalance$\geq$7} is the highest-level-dummy.
\item ThermometerEncoding: Actions on [\textfn{NumInstallTrades$\geq$2}, \textfn{NumInstallTrades$\geq$3}, \textfn{NumInstallTrades$\geq$5}, \textfn{NumInstallTrades$\geq$7}] must preserve thermometer encoding of NumInstallTradesgeq., which can only decrease.Actions can only turn off higher-level dummies that are on, where \textfn{NumInstallTrades$\geq$2} is the lowest-level dummy and \textfn{NumInstallTrades$\geq$7} is the highest-level-dummy.
\end{constraints}

\clearpage
\subsection{Details for the {\normalfont\textds{german}} Dataset}

The \textds{german} dataset was created in 1994 and contains information about loan history, demographics, occupation, payment history, and whether or not somebody is a good customer~\citep{dua2019uci}. Each instance is credit applicant. There are $n = 1,000$ instances and $d=20$ features. The features are all either categorical or discrete. The label a indicates is a loan applicant is ``good'' ($y_i = 1$) or ``bad'' ($y_i = 0$). There are no missing values in the dataset. We renamed some of the features to be indicative of the values they represent. The dataset is self-contained and anonymous, and it includes features describing gender, age, and marital status. 

\paragraph{Actionability Constraints}
\begin{table*}[h]
\centering
\resizebox{0.7\linewidth}{!}{\begin{tabular}{lllllllr}
\toprule
\textheader{Name} & \textheader{Type} & \textheader{LB} & \textheader{UB} & \textheader{Actionability} & \textheader{Sign} & \textheader{Joint Constraints} & \textheader{Partition ID} \\
\midrule
\textfn{Age} & $\mathbb{Z}$ & 19 & 75 & No &  & 0, 4, 12 & 0 \\
\textfn{YearsAtResidence} & $\mathbb{Z}$ & 0 & 7 & Yes & $+$ & 0, 4, 12 & 0 \\
\textfn{YearsEmployed$\geq$1} & $\{0,1\}$ & 0 & 1 & Yes & $+$ & 0, 4, 12 & 0 \\
\textfn{CheckingAcct\_exists} & $\{0,1\}$ & 0 & 1 & Yes & $+$ & 32, 33 & 30 \\
\textfn{CheckingAcct$\geq$0} & $\{0,1\}$ & 0 & 1 & Yes & $+$ & 32, 33 & 30 \\
\textfn{SavingsAcct\_exists} & $\{0,1\}$ & 0 & 1 & Yes & $+$ & 34, 35 & 31 \\
\textfn{SavingsAcct$\geq$100} & $\{0,1\}$ & 0 & 1 & Yes & $+$ & 34, 35 & 31 \\
\textfn{Male} & $\{0,1\}$ & 0 & 1 & No &  & -- & 1 \\
\textfn{Single} & $\{0,1\}$ & 0 & 1 & No &  & -- & 2 \\
\textfn{ForeignWorker} & $\{0,1\}$ & 0 & 1 & No &  & -- & 3 \\
\textfn{LiablePersons} & $\mathbb{Z}$ & 1 & 2 & No &  & -- & 4 \\
\textfn{Housing$=$Renter} & $\{0,1\}$ & 0 & 1 & No &  & -- & 5 \\
\textfn{Housing$=$Owner} & $\{0,1\}$ & 0 & 1 & No &  & -- & 6 \\
\textfn{Housing$=$Free} & $\{0,1\}$ & 0 & 1 & No &  & -- & 7 \\
\textfn{Job$=$Unskilled} & $\{0,1\}$ & 0 & 1 & No &  & -- & 8 \\
\textfn{Job$=$Skilled} & $\{0,1\}$ & 0 & 1 & No &  & -- & 9 \\
\textfn{Job$=$Management} & $\{0,1\}$ & 0 & 1 & No &  & -- & 10 \\
\textfn{CreditAmt$\geq$1000K} & $\{0,1\}$ & 0 & 1 & No &  & -- & 11 \\
\textfn{CreditAmt$\geq$2000K} & $\{0,1\}$ & 0 & 1 & No &  & -- & 12 \\
\textfn{CreditAmt$\geq$5000K} & $\{0,1\}$ & 0 & 1 & No &  & -- & 13 \\
\textfn{CreditAmt$\geq$10000K} & $\{0,1\}$ & 0 & 1 & No &  & -- & 14 \\
\textfn{LoanDuration$\leq$6} & $\{0,1\}$ & 0 & 1 & No &  & -- & 15 \\
\textfn{LoanDuration$\geq$12} & $\{0,1\}$ & 0 & 1 & No &  & -- & 16 \\
\textfn{LoanDuration$\geq$24} & $\{0,1\}$ & 0 & 1 & No &  & -- & 17 \\
\textfn{LoanDuration$\geq$36} & $\{0,1\}$ & 0 & 1 & No &  & -- & 18 \\
\textfn{LoanRate} & $\mathbb{Z}$ & 1 & 4 & No &  & -- & 19 \\
\textfn{HasGuarantor} & $\{0,1\}$ & 0 & 1 & Yes & $+$ & -- & 20 \\
\textfn{LoanRequiredForBusiness} & $\{0,1\}$ & 0 & 1 & No &  & -- & 21 \\
\textfn{LoanRequiredForEducation} & $\{0,1\}$ & 0 & 1 & No &  & -- & 22 \\
\textfn{LoanRequiredForCar} & $\{0,1\}$ & 0 & 1 & No &  & -- & 23 \\
\textfn{LoanRequiredForHome} & $\{0,1\}$ & 0 & 1 & No &  & -- & 24 \\
\textfn{NoCreditHistory} & $\{0,1\}$ & 0 & 1 & No &  & -- & 25 \\
\textfn{HistoryOfLatePayments} & $\{0,1\}$ & 0 & 1 & No &  & -- & 26 \\
\textfn{HistoryOfDelinquency} & $\{0,1\}$ & 0 & 1 & No &  & -- & 27 \\
\textfn{HistoryOfBankInstallments} & $\{0,1\}$ & 0 & 1 & Yes & $+$ & -- & 28 \\
\textfn{HistoryOfStoreInstallments} & $\{0,1\}$ & 0 & 1 & Yes & $+$ & -- & 29 \\
\bottomrule
\end{tabular}

}
\caption{
Separable Actionability Constraints for the processed \textds{german} dataset. 
\textbf{Type} indicates the feature type ($\Z$ for integer, $\{0,1\}$ for binary). \textbf{LB}, \textbf{UB} are the lower and upper bounds for the feature.
\textbf{Actionability} indicates whether the feature is globally actionable.
\textbf{Sign} indicates if the feature can only increase (+) or decrease (-).
\textbf{Joint Constraints} are a list non-separable constraint indices it is tied to (if any).
\textbf{Partition ID} indicates which partition the feature belongs to.
}
\label{Tab::germanASet}
\end{table*}

The joint actionability constraints listed in \cref{Tab::germanASet} include:
\begin{constraints}
\item DirectionalLinkage: Actions on \textfn{YearsAtResidence} will induce to actions on ['\textfn{Age}'].Each unit change in \textfn{YearsAtResidence} leads to:1.00-unit change in \textfn{Age}
\item DirectionalLinkage: Actions on \textfn{YearsEmployed$\geq$1} will induce to actions on ['\textfn{Age}'].Each unit change in \textfn{YearsEmployed$\geq$1} leads to:1.00-unit change in \textfn{Age}
\item ThermometerEncoding: Actions on [\textfn{CheckingAcctexists}, \textfn{CheckingAcct$\geq$0}] must preserve thermometer encoding of CheckingAcct., which can only increase.Actions can only turn on higher-level dummies that are off, where \textfn{CheckingAcctexists} is the lowest-level dummy and \textfn{CheckingAcct$\geq$0} is the highest-level-dummy.
\item ThermometerEncoding: Actions on [\textfn{SavingsAcctexists}, \textfn{SavingsAcct$\geq$100}] must preserve thermometer encoding of SavingsAcct., which can only increase.Actions can only turn on higher-level dummies that are off, where \textfn{SavingsAcctexists} is the lowest-level dummy and \textfn{SavingsAcct$\geq$100} is the highest-level-dummy.
\end{constraints}

\clearpage
\subsection{Details for the {\normalfont\textds{givemecredit}} Dataset}

The \textds{givemecredit} dataset is used to determine whether a loan should be given or denied~\citep{data2018givemecredit}. The label indicates whether someone was 90 days past due in the two years following data collection. Delinquency refers to a debt with an overdue payment; this dataset is used to predict if someone will experience financial distress in the next two years.It contains information about $n=120,268$ loan recipients, and each instance represents a borrower. There are $d=10$ features before preprocessing. Here, the label is $y_i = 0$ if an applicant has had a serious delinquency in two years and $y_i$ otherwise. The data is self-contained and anonymous, and it contains features describing age, income, and the number of dependents.

\paragraph{Actionability Constraints}
\begin{table*}[h]
\centering
\resizebox{0.7\linewidth}{!}{\begin{tabular}{lllllllr}
\toprule
\textheader{Name} & \textheader{Type} & \textheader{LB} & \textheader{UB} & \textheader{Actionability} & \textheader{Sign} & \textheader{Joint Constraints} & \textheader{Partition ID} \\
\midrule
\textfn{CreditLineUtilization$\geq$10.0} & $\{0,1\}$ & 0 & 1 & Yes &  & 12, 13, 14, 15, 16 & 10 \\
\textfn{CreditLineUtilization$\geq$20.0} & $\{0,1\}$ & 0 & 1 & Yes &  & 12, 13, 14, 15, 16 & 10 \\
\textfn{CreditLineUtilization$\geq$50.0} & $\{0,1\}$ & 0 & 1 & Yes &  & 12, 13, 14, 15, 16 & 10 \\
\textfn{CreditLineUtilization$\geq$70.0} & $\{0,1\}$ & 0 & 1 & Yes &  & 12, 13, 14, 15, 16 & 10 \\
\textfn{CreditLineUtilization$\geq$100.0} & $\{0,1\}$ & 0 & 1 & Yes &  & 12, 13, 14, 15, 16 & 10 \\
\textfn{MonthlyIncome$\geq$3K} & $\{0,1\}$ & 0 & 1 & Yes & $+$ & 9, 10, 11 & 9 \\
\textfn{MonthlyIncome$\geq$5K} & $\{0,1\}$ & 0 & 1 & Yes & $+$ & 9, 10, 11 & 9 \\
\textfn{MonthlyIncome$\geq$10K} & $\{0,1\}$ & 0 & 1 & Yes & $+$ & 9, 10, 11 & 9 \\
\textfn{AnyRealEstateLoans} & $\{0,1\}$ & 0 & 1 & Yes & $+$ & 17, 18 & 11 \\
\textfn{MultipleRealEstateLoans} & $\{0,1\}$ & 0 & 1 & Yes & $+$ & 17, 18 & 11 \\
\textfn{AnyCreditLinesAndLoans} & $\{0,1\}$ & 0 & 1 & Yes & $+$ & 19, 20 & 12 \\
\textfn{MultipleCreditLinesAndLoans} & $\{0,1\}$ & 0 & 1 & Yes & $+$ & 19, 20 & 12 \\
\textfn{Age$\leq$24} & $\{0,1\}$ & 0 & 1 & No &  & -- & 0 \\
\textfn{Age\_bt\_25\_to\_30} & $\{0,1\}$ & 0 & 1 & No &  & -- & 1 \\
\textfn{Age\_bt\_30\_to\_59} & $\{0,1\}$ & 0 & 1 & No &  & -- & 2 \\
\textfn{Age$\geq$60} & $\{0,1\}$ & 0 & 1 & No &  & -- & 3 \\
\textfn{NumberOfDependents$=$0} & $\{0,1\}$ & 0 & 1 & No &  & -- & 4 \\
\textfn{NumberOfDependents$=$1} & $\{0,1\}$ & 0 & 1 & No &  & -- & 5 \\
\textfn{NumberOfDependents$\geq$2} & $\{0,1\}$ & 0 & 1 & No &  & -- & 6 \\
\textfn{NumberOfDependents$\geq$5} & $\{0,1\}$ & 0 & 1 & No &  & -- & 7 \\
\textfn{DebtRatio$\geq$1} & $\{0,1\}$ & 0 & 1 & Yes & $+$ & -- & 8 \\
\textfn{HistoryOfLatePayment} & $\{0,1\}$ & 0 & 1 & No &  & -- & 13 \\
\textfn{HistoryOfDelinquency} & $\{0,1\}$ & 0 & 1 & No &  & -- & 14 \\
\bottomrule
\end{tabular}

}
\caption{
Separable Actionability Constraints for the processed \textds{givemecredit} dataset. 
\textbf{Type} indicates the feature type ($\Z$ for integer, $\{0,1\}$ for binary). \textbf{LB}, \textbf{UB} are the lower and upper bounds for the feature.
\textbf{Actionability} indicates whether the feature is globally actionable.
\textbf{Sign} indicates if the feature can only increase (+) or decrease (-).
\textbf{Joint Constraints} are a list non-separable constraint indices it is tied to (if any).
\textbf{Partition ID} indicates which partition the feature belongs to.
}
\label{Tab::givemeASet}
\end{table*}

The joint actionability constraints listed in \cref{Tab::givemeASet} include:
\begin{constraints}
\item ThermometerEncoding: Actions on [\textfn{MonthlyIncome$\geq$3K}, \textfn{MonthlyIncome$\geq$5K}, \textfn{MonthlyIncome$\geq$10K}] must preserve thermometer encoding of MonthlyIncomegeq., which can only increase.Actions can only turn on higher-level dummies that are off, where \textfn{MonthlyIncome$\geq$3K} is the lowest-level dummy and \textfn{MonthlyIncome$\geq$10K} is the highest-level-dummy.
\item ThermometerEncoding: Actions on [\textfn{CreditLineUtilization$\geq$10.0}, \textfn{CreditLineUtilization$\geq$20.0}, \textfn{CreditLineUtilization$\geq$50.0}, \textfn{CreditLineUtilization$\geq$70.0}, \textfn{CreditLineUtilization$\geq$100.0}] must preserve thermometer encoding of CreditLineUtilizationgeq., which can only decrease.Actions can only turn off higher-level dummies that are on, where \textfn{CreditLineUtilization$\geq$10.0} is the lowest-level dummy and \textfn{CreditLineUtilization$\geq$100.0} is the highest-level-dummy.
\item ThermometerEncoding: Actions on [\textfn{AnyRealEstateLoans}, \textfn{MultipleRealEstateLoans}] must preserve thermometer encoding of continuousattribute., which can only decrease.Actions can only turn off higher-level dummies that are on, where \textfn{AnyRealEstateLoans} is the lowest-level dummy and \textfn{MultipleRealEstateLoans} is the highest-level-dummy.
\item ThermometerEncoding: Actions on [\textfn{AnyCreditLinesAndLoans}, \textfn{MultipleCreditLinesAndLoans}] must preserve thermometer encoding of continuousattribute., which can only decrease.Actions can only turn off higher-level dummies that are on, where \textfn{AnyCreditLinesAndLoans} is the lowest-level dummy and \textfn{MultipleCreditLinesAndLoans} is the highest-level-dummy.
\end{constraints}

\clearpage
\subsection{Overview of Model Performance}
We include the performance of the classifiers used in \cref{Sec::Experiments}.

\begin{table*}[h!]
    \centering
    \resizebox{0.6\linewidth}{!}{\begin{tabular}{lllllll}
\multicolumn{1}{c}{ } & \multicolumn{2}{c}{\LR{}} & \multicolumn{2}{c}{\XGB{}} & \multicolumn{2}{c}{\RF{}} \\
\cmidrule(l{3pt}r{3pt}){2-3} \cmidrule(l{3pt}r{3pt}){4-5} \cmidrule(l{3pt}r{3pt}){6-7}
Dataset & Train & Test & Train & Test & Train & Test\\
\midrule
\ficoinfo{} & \cell{r}{0.772} & \cell{r}{0.788} & \cell{r}{0.859} & \cell{r}{0.785} & \cell{r}{0.780} & \cell{r}{0.790}\\
\midrule

\germaninfo{} & \cell{r}{0.819} & \cell{r}{0.760} & \cell{r}{0.971} & \cell{r}{0.794} & \cell{r}{0.828} & \cell{r}{0.766}\\
\midrule

\givemecreditinfo{} & \cell{r}{0.841} & \cell{r}{0.844} & \cell{r}{0.875} & \cell{r}{0.793} & \cell{r}{0.864} & \cell{r}{0.835}\\
\bottomrule
\end{tabular}
}
    \caption{Train and Test AUC for models across all datasets. We optimized the model's hyperparameters through randomized search and divided the data into training and testing sets at an 80\% and 20\% ratio.}
    \label{Tab:ModelPerf}
\end{table*}

\section{Supplementary Experiment Results}
\subsection{Responsiveness of Explanations for Random Forests}
\label{Appendix::RFResults}

\begin{table*}[h]
    \centering
    \resizebox{\linewidth}{!}{
    
\begin{tabular}{ll*{5}r}
\multicolumn{2}{c}{ } & \multicolumn{2}{c}{\regular{}} & \multicolumn{2}{c}{\filtered{}} & \multicolumn{1}{c}{ } \\
\cmidrule(l{3pt}r{3pt}){3-4} \cmidrule(l{3pt}r{3pt}){5-6}
Dataset & Metrics & \lime{} & \shap{} & \limeaa{} & \shapaa{} & \resp\\
\midrule
\ficoinfolong{} & \dmeMetrics{} & \cell{r}{100.0\%\\\color{bad}{86.5\%}\\13.5\%\\0.0\%\\4.0} & \cell{r}{100.0\%\\\color{bad}{78.2\%}\\21.8\%\\0.0\%\\4.0} & \cell{r}{100.0\%\\\color{bad}{77.1\%}\\22.9\%\\0.0\%\\4.0} & \cell{r}{100.0\%\\\color{bad}{76.7\%}\\23.3\%\\0.5\%\\4.0} & \cell{r}{31.7\%\\0.0\%\\100.0\%\\\cellcolor{good}{\textbf{100.0\%}}\\2.4}\\
\midrule

\germaninfolong{} & \dmeMetrics{} & \cell{r}{100.0\%\\\color{bad}{100.0\%}\\0.0\%\\0.0\%\\4.0} & \cell{r}{100.0\%\\\color{bad}{89.1\%}\\10.9\%\\0.0\%\\4.0} & \cell{r}{100.0\%\\\color{bad}{76.6\%}\\23.4\%\\0.0\%\\4.0} & \cell{r}{100.0\%\\\color{bad}{64.6\%}\\35.4\%\\0.0\%\\4.0} & \cell{r}{48.0\%\\0.0\%\\100.0\%\\\cellcolor{good}{\textbf{100.0\%}}\\2.2}\\
\midrule

\givemecreditinfolong{} & \dmeMetrics{} & \cell{r}{100.0\%\\\color{bad}{56.5\%}\\43.5\%\\0.0\%\\4.0} & \cell{r}{100.0\%\\\color{bad}{26.8\%}\\73.2\%\\0.5\%\\4.0} & \cell{r}{100.0\%\\\color{bad}{28.4\%}\\71.6\%\\1.4\%\\4.0} & \cell{r}{100.0\%\\\color{bad}{21.0\%}\\79.0\%\\11.4\%\\4.0} & \cell{r}{93.2\%\\0.0\%\\100.0\%\\\cellcolor{good}{\textbf{100.0\%}}\\2.9}\\
\bottomrule
\end{tabular}

    }
    \caption{Responsiveness of feature-highlighting explanations for \RF{} for all methods and datasets. We generate explanations that highlight up to 4 top-scoring features from a given method. We report the proportion of individuals receiving an explanation (\emph{\% Presented with Explanations}) and the mean number of features in each explanation (\emph{\# Features Highlighted}). We also show the proportion of instances where all features are unresponsive (\emph{\% All Features Unresponsive}) highlighting {\color{bad}{positive values}}; at least one feature is responsive (\emph{\% At Least 1 Feature Responsive}), or all features are responsive (\emph{\% All Features Responsive}) highlighting the \textbf{\hl{best value}}.}
    \label{Table::RFResults}
\end{table*}

\clearpage
\subsection{Feature Responsiveness Rankings}
\label{Appendix::TopKPlots}

We include a plot to show how responsive features are at different rankings by \lime{}, \shap{}, \limeaa{}, \shapaa{} and \resp{} for each dataset. For every denied individual, we rank features by their absolute feature importance score returned by these methods. We exclude features with 0 attribution from the rankings.

The plots below show the \% of times where the feature at each rank are responsive (i.e., feature has \resp{} > 0). It allows us to visualize and compare how often these methods assign high attribution to responsive features.

\newcommand{\rankplotind}[3]{
\begin{wrapfigure}{}{0.45\textwidth}
    \centering
    \resizebox{\linewidth}{!}{
    \includegraphics[]{images/rank/#1_complex_nD_#2_#3_plot_rank.pdf}
    }
\end{wrapfigure}
}

\newcommand{\mdl}{temp}
\newcommand{\rankplotdouble}[2]{
  \ifthenelse{\equal{#2}{logreg}}
    {\renewcommand{\mdl}{\LR{} }}
    {
      \ifthenelse{\equal{#2}{xgb}}
        {\renewcommand{\mdl}{\XGB{} }}
        {
          \ifthenelse{\equal{#2}{rf}}
            {\renewcommand{\mdl}{\RF{} }}
            {\renewcommand{\mdl}{Unknown Model}}
        }
    }
\begin{figure}[H]
\centering
\resizebox{0.38\linewidth}{!}{
    \includegraphics[]{images/rank/#1_complex_nD_#2_SHAP_plot_rank.pdf}
}
\resizebox{0.38\linewidth}{!}{
    \includegraphics[]{images/rank/#1_complex_nD_#2_LIME_plot_rank.pdf}
}
\caption{
Responsiveness of features for individuals who are denied credit by the \mdl model on the \textds{#1} dataset according to absolute feature attribution rank using the original feature attribution method, its action-aware variant and \resp{}. For each method, we report the proportion of individuals with at least one responsive intervention on a feature with the $k$-th largest score ($k$-th ranked feature).
Features must have non-zero score to be included in a ``rank.''
}
\end{figure}
}

\subsubsection{{\normalfont\textds{heloc}}}
\rankplotdouble{fico}{logreg}
\rankplotdouble{fico}{xgb}
\rankplotdouble{fico}{rf}

\subsubsection{{\normalfont\textds{german}}}
\rankplotdouble{german}{logreg}
\rankplotdouble{german}{xgb}
\rankplotdouble{german}{rf}

\subsubsection{{\normalfont\textds{givemecredit}}}
\rankplotdouble{givemecredit}{logreg}
\rankplotdouble{givemecredit}{xgb}
\rankplotdouble{givemecredit}{rf}

\section{Supplementary Case Study Details}
\label{Appendix::DemoDetails}
\subsection{Actionability Constraints}
\begin{table*}[h]
\centering
\resizebox{0.7\linewidth}{!}{\begin{tabular}{lllllllr}
\toprule
\textheader{Name} & \textheader{Type} & \textheader{LB} & \textheader{UB} & \textheader{Actionability} & \textheader{Sign} & \textheader{Joint Constraints} & \textheader{Partition ID} \\
\midrule
\textfn{Age} & $\mathbb{Z}$ & 21 & 103 & No &  & 0, 8, 10 & 0 \\
\textfn{HistoryOfLatePaymentInPast2Years} & $\{0,1\}$ & 0 & 1 & Yes & $+$ & 0, 8, 10 & 0 \\
\textfn{HistoryOfDelinquencyInPast2Years} & $\{0,1\}$ & 0 & 1 & Yes & $+$ & 0, 8, 10 & 0 \\
\textfn{NumberRealEstateLoansOrLines} & $\mathbb{Z}$ & 0 & 100 & Yes & $+$ & 5, 6 & 5 \\
\textfn{NumberOfOpenCreditLinesAndLoans} & $\mathbb{Z}$ & 0 & 100 & Yes & $+$ & 5, 6 & 5 \\
\textfn{NumberOfDependents} & $\mathbb{Z}$ & 0 & 20 & No &  & -- & 1 \\
\textfn{DebtRatio} & $\mathbb{R}$ & 0.0 & 61106.5 & Yes &  & -- & 2 \\
\textfn{MonthlyIncome} & $\mathbb{Z}$ & 0 & 3008750 & Yes &  & -- & 3 \\
\textfn{CreditLineUtilization} & $\mathbb{R}$ & 0.0 & 50708.0 & Yes &  & -- & 4 \\
\textfn{HistoryOfLatePayment} & $\{0,1\}$ & 0 & 1 & No &  & -- & 6 \\
\textfn{HistoryOfDelinquency} & $\{0,1\}$ & 0 & 1 & No &  & -- & 7 \\
\bottomrule
\end{tabular}

}
\caption{
Separable Actionability Constraints for the processed continuous \textds{givemecredit} dataset. 
\textbf{Type} indicates the feature type ($\Z$ for integer, $\{0,1\}$ for binary). \textbf{LB}, \textbf{UB} are the lower and upper bounds for the feature.
\textbf{Actionability} indicates whether the feature is globally actionable.
\textbf{Sign} indicates if the feature can only increase (+) or decrease (-).
\textbf{Joint Constraints} are a list non-separable constraint indices it is tied to (if any).
\textbf{Partition ID} indicates which partition the feature belongs to.
}
\label{Tab::givemectsASet}
\end{table*}

The joint actionability constraints listed in \cref{Tab::givemectsASet} include:
\begin{constraints}
\item DirectionalLinkage: Actions on \textfn{NumberRealEstateLoansOrLines} will induce to actions on ['\textfn{NumberOfOpenCreditLinesAndLoans}'].Each unit change in \textfn{NumberRealEstateLoansOrLines} leads to:1.00-unit change in \textfn{NumberOfOpenCreditLinesAndLoans}
\item DirectionalLinkage: Actions on \textfn{HistoryOfLatePaymentInPast2Years} will induce to actions on ['\textfn{Age}'].Each unit change in \textfn{HistoryOfLatePaymentInPast2Years} leads to:2.00-unit change in \textfn{Age}
\item DirectionalLinkage: Actions on \textfn{HistoryOfDelinquencyInPast2Years} will induce to actions on ['\textfn{Age}'].Each unit change in \textfn{HistoryOfDelinquencyInPast2Years} leads to:2.00-unit change in \textfn{Age}
\end{constraints}

\subsection{Model Performance}
\begin{table}[h]
    \centering
\begin{tabular}{lll}
\multicolumn{1}{c}{ } & \multicolumn{2}{c}{\XGB{}} \\
\cmidrule(l{3pt}r{3pt}){2-3} 
Dataset & Train & Test\\
\midrule
\cell{l}{\textds{givemecredit}\\$n = 120,268$\\$d=11$\\ \citet{data2018givemecredit}} & \cell{r}{0.937} & \cell{r}{0.830}\\
\bottomrule
\end{tabular}
\caption{Model Performance of \XGB{} model on the \textds{givemecredit} dataset for \cref{Sec::Demos}}
\label{tab:my_label}
\end{table}

\end{document}